\documentclass[twoside]{article}

\usepackage[accepted]{aistats2023}
\usepackage[round]{natbib}

\usepackage[dvipsnames]{xcolor}
\usepackage{booktabs}
\usepackage{amsmath,amsthm,amssymb,amsfonts}
\usepackage{algorithm,algorithmic}
\usepackage{dsfont}
\usepackage{enumitem}
\usepackage{macros}

\newtheorem{definition}{Definition}
\newtheorem{assumption}{Assumption}
\newtheorem{proposition}{Proposition}
\newtheorem{corollary}{Corollary}
\newtheorem{theorem}{Theorem}
\newtheorem{lemma}{Lemma}
\newtheorem{remark}{Remark}

\usepackage[colorinlistoftodos]{todonotes}

\begin{document}

%

%

\twocolumn[

\aistatstitle{On the Complexity of Representation Learning in Contextual Linear Bandits}

\aistatsauthor{ Andrea Tirinzoni \And Matteo Pirotta \And  Alessandro Lazaric }

\aistatsaddress{ Meta AI \And  Meta AI \And Meta AI } ]

\begin{abstract}

In contextual linear bandits, the reward function is assumed to be a linear combination of an unknown reward vector and a given embedding of context-arm pairs. In practice, the embedding is often learned at the same time as the reward vector, thus leading to an \emph{online representation learning} problem. Existing approaches to representation learning in contextual bandits are either very generic (e.g., model-selection techniques or algorithms for learning with arbitrary function classes) or specialized to particular structures (e.g., nested features or representations with certain spectral properties). As a result, the understanding of the cost of representation learning in contextual linear bandit is still limited. In this paper, we take a systematic approach to the problem and provide a comprehensive study through an instance-dependent perspective. We show that representation learning is fundamentally more complex than linear bandits (i.e., learning with a given representation). In particular, learning with a given set of representations is never simpler than learning with the \emph{worst} realizable representation in the set, while we show cases where it can be arbitrarily harder. We complement this result with an extensive discussion of how it relates to existing literature and we illustrate positive instances where representation learning is as complex as learning with a fixed representation and where sub-logarithmic regret is achievable.
\end{abstract}

\section{Introduction}\label{sec:introduction}

Stochastic contextual linear bandits (CLBs) focus on the interplay between exploration and exploitation when the reward $f^\star(x,a)$ of each context-arm pair $(x,a) \in \cX \times \cA$ is a linear function of a known feature map $\phi^\star : \cX \times \cA \to \mathbb{R}^{d_{\phi^\star}}$ and an unknown parameter $\theta^\star$. CLBs have been widely studied due to their broad applicability and strong theoretical guarantees~\citep[e.g.,][and references therein]{lattimore2020bandit}. 
Unfortunately, the assumption that a realizable linear representation is known is often violated in real applications, where one only observes raw context-arm data and a suitable representation has to be learned online. Representation learning in CLBs relaxes this assumption by providing the learner with
a set of representations $\Phi = \{\phi : \cX \times \cA \to \mathbb{R}^{d_\phi}\}$ (e.g., a neural network) 
among which a realizable one exists (i.e., $\phi^\star \in \Phi$). 

Representation learning can be viewed as a special case of learning with a general realizable function class (i.e., $\cF_\Phi := \{f(\cdot,\cdot) = \phi(\cdot,\cdot)^\transp\theta \mid \phi\in\Phi,\theta\in\bR^{d_{\phi}}\}$), which has been extensively studied in the literature~\citep[e.g.,][]{AgarwalHKLLS14,Foster2020beyond,SimchiLevi2020falcon} with algorithms achieving $O(\sqrt{AT \log(|\cF_\Phi|)})$ worst-case regret, where $|\cF_\Phi|$ is the covering number of $\cF_\Phi$. However, these algorithms do not explicitly leverage the bi-linear structure of the function class $\cF_\Phi$.
Another direction is to leverage model-selection techniques. While generic model-selection approaches~\citep[e.g.][]{abbasiyadkori2020regret,pacchiano2020stochcorral,CutkoskyDDGPP21} can be directly applied when $\Phi$ is finite, more specialized techniques can be used when $\Phi$ has additional structure (e.g., nested features~\citep{Foster2019nested}). Interestingly, some of these algorithms~\citep[e.g.,][]{Foster2019nested,CutkoskyDDGPP21,ghosh2021problem} achieve regret guarantees matching the performance of the best representation in the set, up to a representation learning cost that depends on the number of representations $|\Phi|$, the problem horizon $T$, or other quantities specific to the structure of $\Phi$. Nonetheless, these results are worst-case in nature and general model-selection algorithms are limited by an unavoidable $\Omega(\sqrt{T})$ regret~\citep{pacchiano2020stochcorral}, which may hinder them from fully exploiting the structure of $\Phi$ and achieve instance-optimal performance (e.g., logarithmic regret). Alternatively, \citet{Papini2021leader} and~\citet{banditsrl} proposed specialized representation learning algorithms that exploit the bi-linear structure of $\mathcal F_{\Phi}$ to obtain the
%
\emph{instance-dependent} regret bound of the best unknown realizable representation up to a logarithmic factor in $|\Phi|$. Furthermore, they showed that constant regret is achievable (i.e., after a finite time $\tau$ the algorithm only plays optimal arms) when a realizable representation satisfies a certain spectral property. However, these results rely on the strong assumption that either all the representations in $\Phi$ are realizable or any misspecified representation can be identified by playing any sequence of arms.

In this paper, we focus on the following question:
\begin{center}
    \textit{What is the cost of representation learning compared to a CLB with a given representation?}
\end{center}

In order to address this question, we first provide a systematic analysis of representation learning in CLBs through an instance-dependent lens. By specializing existing results, we derive an instance-dependent lower bound on the regret of any ``good'' representation learning algorithm which shows that the asymptotic regret must be at least $\cC(f^\star,\cF_\Phi)\log(T)$, where $\cC(f^\star,\cF_\Phi)$ is a complexity measure depending both on the reward function $f^\star$ and the given set of representations $\Phi$. Moreover, this complexity is tight, as there exist algorithms attaining $\cC(f^\star,\cF_\Phi)\log(T)$ regret in the large $T$ regime. This instance-dependent view allows us to have a more fine grained comparison to CLBs with a given representation, thus providing insights on the complexity of representation learning that may remain ``hidden'' in worst-case studies. 



Leveraging this lower bound we are then able to derive the following results: \textbf{(1)} We show that the regret of representation learning is never smaller than the regret of learning with the \emph{worst} realizable representation in the set, i.e., $\cC(f^\star,\cF_\Phi) \geq \sup_{\phi \in \Phi, \text{realizable}}\cC(f^\star,\cF_{\{\phi\}})$. This reveals a fundamental limit to representation learning,  showing that it is impossible to adapt to representations with better complexity. Surprisingly, this result holds even for instances $f^\star$ where all representations $\phi\in\Phi$ are realizable. Indeed, this is due to a subtle but crucial effect of representation learning: as in general all  representations $\phi\in\Phi$ may be misspecified for some of the reward functions $f'\in\mathcal{F}_{\Phi}$, an algorithm needs to be robust to such misspecification and it cannnot fully adapt to cases that are favorable for some representations. \textbf{(2)} We further strengthen this result by showing examples where the inequality is strict and the gap arbitrarily large. In particular, we construct instances where all representations are realizable and have small dimensionality and yet the regret can be as large as learning with ``tabular'' features assigning a distinct dimension to each context-arm pair. \textbf{(3)} We characterize favorable instances where misspecified representations in $\Phi$ can be discarded without increasing the regret so that $\cC(f^\star,\cF_\Phi) = \sup_{\phi \in \Phi, \text{realizable}}\cC(f^\star,\cF_{\{\phi\}})$.

Finally, we instantiate our analysis in widely studied representation structures (e.g., tabular, nested features, features with spectral properties, and the special case where all representations are realizable) and provide novel insights on the complexity of representation learning in these settings. 

\section{Preliminaries}\label{sec:preliminaries}

We consider a \emph{stochastic contextual bandit} problem with a finite set of contexts $\cX$ and a finite set of arms $\cA$. Let $X := |\cX|$ and $A := |\cA|$. At each time step $t\in\mathbb{N}$, the learner first observes a context $x_t\in\cX$ drawn i.i.d.\ from a distribution $\rho$\footnote{We assume $\rho$ to be full-support over $\cX$ w.l.o.g.}, it selects an arm $a_t\in\cA$, and it receives a scalar reward drawn from a Gaussian distribution with mean $f^\star(x_t,a_t)$ and unit variance. 

Let $\Phi$ be a set of \emph{representations}, where each $\phi\in\Phi$ is a $d_\phi$-dimensional \emph{feature map} $\phi: \cX \times \cA \rightarrow \mathbb{R}^{d_\phi}$. We define the associated function class $\cF_\Phi := \{f(\cdot,\cdot) = \phi(\cdot,\cdot)^\transp\theta \mid \phi\in\Phi,\theta\in\bR^{d_\phi}\}$.
The set $\Phi$ and function class $\cF_\Phi$ are realizable when:
\begin{assumption}[Realizability]\label{asm:realizability}
There exist $\phi^\star\in\Phi$ and $\theta^\star \in \mathbb{R}^{d^\star}$, where $d^\star := d_{\phi^\star}$,  such that
\begin{align*}
    f^\star(x,a) = \phi^\star(x,a)^\transp \theta^\star \quad \forall x\in\cX,a\in\cA.
\end{align*}
\end{assumption}
This assumption is required only for a representation $\phi^\star\in\Phi$ (which is said to be \emph{realizable}), while, for any $\phi\neq\phi^\star$, the approximation error $\max_{x,a}|f^\star(x,a) - \phi(x,a)^\transp\theta|$ may be non-zero for any $\theta$, meaning that $f^\star$ cannot be approximated as a linear function of $\phi$. In this case, we shall say that representation $\phi$ is \emph{misspecified}.

\paragraph{Learning problem.}
We consider the problem of (bi-linear) \emph{representation learning} for regret minimization. 

\begin{definition}[Representation learning problem $(f^\star, \cF_\Phi)$]\label{def:replearning}
    Consider an unknown stochastic contextual bandit problem with reward function $f^\star$. The learner is provided only with a set of representations $\Phi$ (equiv. function class $\cF_\Phi$) satisfying Assumption \ref{asm:realizability} ($\phi^\star$ unknown) and it aims at minimizing the cumulative regret over $T$ steps, %
    \begin{align}\label{eq:regret}
        R_T(f^\star) := \sum_{t=1}^T \left(\max_{a\in\cA}f^\star(x_t,a) - f^\star(x_t,a_t)\right).
    \end{align}
\end{definition}

When $\Phi =\{\phi^\star\}$, the learning problem $(f^\star, \cF_{\{\phi^\star\}})$ is known as stochastic \emph{contextual linear bandit} (CLB), where the learner 
%
knows the realizable representation $\phi^\star$, while in representation learning the learner needs to learn within the realizable \emph{non-linear} function class $\cF_\Phi$. Note also that $\Phi$ may be an \emph{infinite uncountable} set. 

\paragraph{Notation}

We use $M^{\dagger}$ to denote the pseudo-inverse of a matrix $M\in\bR^{n\times m}$, while $\im(M)$ and $\ker(M)$ denote its column and null spaces, respectively.
For a vector $v\in\mathbb{R}^d$ and a matrix $M \in \mathbb{R}^{d\times d}$, we let $\|v\|_M^2 := v^\transp M v$. We use $\pi^\star_{f^\star}(x) := \argmax_{a\in\cA}f^\star(x,a)$ to denote the optimal arm for context $x$ when facing a problem with reward $f
^\star$. We assume $\pi^\star_{f^\star}(x)$ to be unique for all $x$.
We define the sub-optimality gap of arm $a\in\cA$ for context $x\in\cX$ as $\Delta_{f^\star}(x,a) := f^\star(x,\pi^\star_{f^\star}(x)) - f^\star(x,a)$.
Note that, under Assumption~\ref{asm:realizability}, we have $ \Delta_{f^\star}(x,a) = z^\star_{\phi^\star}(x,a)^\transp \theta^\star$, where we call $z^\star_\phi(x,a) := \phi(x,\pi_{f^\star}^\star(x)) - \phi(x,a)$ the \emph{feature gap}.

We will often use a matrix notation for all quantities. We denote by $f^\star \in \bR^{XA}$ a vectorized reward function and by $D_\eta := \diag(\{\eta(x,a)\}_{x\in\cX,a\in\cA})$ the $XA \times XA$ matrix containing a function $\eta : \cX\times\cA \rightarrow [0,\infty)$. For any $\phi\in\Phi$, let $F_\phi \in \bR^{XA \times d_\phi}$ be the matrix containing the feature vectors $\{\phi(x,a)\}_{x\in\cX,a\in\cA}$ as rows and $V_\eta(\phi) := F_\phi^\transp D_\eta F_\phi = \sum_{x,a}\eta(x,a)\phi(x,a)\phi(x,a)^\transp$. Note that $f^\star = F_{\phi^\star}\theta^\star$. 

Using this notation, $\| f^\star - F_\phi \theta\|_{D_\eta}^2$ is exactly the mean square error of the function $\phi(\cdot,\cdot)^\transp \theta$ in predicting $f^\star$ when the learner has $\eta(x,a)$ samples from each $(x,a)$. We define  $\theta^\star_\eta(\phi) := \argmin_{\theta\in\bR^{d_\phi}} \| f^\star - F_\phi \theta\|_{D_\eta}^2$ as the best fit for the reward parameter using representation $\phi$. By standard regression theory, it is easy to show that $\theta^\star_\eta(\phi) = V_\eta(\phi)^\dagger \sum_{x,a}\eta(x,a)\phi(x,a)f^\star(x,a)$. Similarly, the quantity $\| f^\star - F_\phi \theta^\star_\eta(\phi)\|_{D_\eta}^2$ is related to the \emph{misspecification} of representation $\phi$: it is zero for all $\eta$ if $\phi$ is realizable, while it is positive for at least one $\eta$ if $\phi$ is misspecified.
\section{Instance-dependent Regret Lower Bound}\label{sec:lower-bounds}

We start by stating a novel asymptotic regret lower bound for the representation learning problem $(f^\star, \cF_{\Phi})$ (see Definition~\ref{def:replearning}).
Let $\mathfrak{A}$ be any bandit strategy, i.e., a sequence $\{\mathfrak{A}_t \}_{t \geq 1}$ where each $\mathfrak{A}_t : (\cX \times \cA \times \bR)^{t-1} \times \cX \rightarrow \cA$ is a measurable mapping w.r.t.\ the history up to time step $t-1$. We say that a $\mathfrak{A}$ is \emph{uniformly good} on a function class $\cF$ if $\bE_{f}^\mathfrak{A} \big[ R_T(f) \big] = o(T^\alpha)$ for any $\alpha>0$ and any $f\in\cF$\footnote{Our analysis easily extends to the weaker notion of uniformly good algorithm requiring $O(T^\alpha)$ regret on all $f\in\cF$ only for some $\alpha\in(0,1)$. In this case, the stated lower bound remains the same as in Theorem \ref{th:lower-bound-repr} up to a factor $1-\alpha$ \citep{tirinzoni2021fully}.}, where $\bE_f^{\mathfrak{A}}$ denotes the expectation under algorithm $\mathfrak{A}$ in a contextual bandit problem with reward function $f\in\cF$. 
\begin{theorem}\label{th:lower-bound-repr}
    Let $\mathfrak{A}$ be a \textit{uniformly good} strategy on the class $\cF_\Phi$ and suppose that $\pi^\star_{f^\star}$ is unique. Then,
    \begin{equation*}
        \liminf_{T \rightarrow \infty}\frac{\bE_{f^\star}^\mathfrak{A} \big[ R_T(f^\star) \big]}{\log(T)} \geq \cC(f^\star,\cF_\Phi),
    \end{equation*} 
    where $\cC(f^\star,\cF_\Phi)$ is the value of the optimization problem
    \begin{equation*}
        \begin{aligned}
        &\underset{\{\eta(x,a)\} \geq 0}{\inf} \sum_{x\in\cX}\sum_{a\in\cA}\eta(x,a)\Delta_{f^\star}(x,a) \quad \mathrm{s.t.}
        \\
        & \inf_{\phi\in\Phi} \min_{x,a \neq \pi^\star_{f^\star}(x)}  \left( \| f^\star - F_\phi \theta^\star_\eta(\phi)\|_{D_\eta}^2 + c_{x,a}^\eta(f^\star, \phi) \right) \geq 2,
        \end{aligned}
    \end{equation*}
    with
    \begin{align*}
        c_{x,a}^\eta(f^\star, \phi) = \begin{cases}
            0 & \hspace{-0.2cm}\text{if } z^\star_\phi(x,a)^\transp \theta^\star_\eta(\phi) \leq 0,\\
            0 & \hspace{-0.2cm}\text{if } z^\star_\phi(x,a)\notin\im(V_\eta(\phi)),\\
            \frac{(z^\star_\phi(x,a)^\transp \theta^\star_\eta(\phi))^2}{\|z^\star_\phi(x,a)\|_{V_\eta(\phi)^\dagger}^2} & \hspace{-0.1cm}\text{otherwise}.\\
        \end{cases}
    \end{align*}
\end{theorem}
The proof (see Appendix \ref{app:lower-bound}) builds on the asymptotic regret lower bound for contextual bandits with general function classes (a.k.a.\ structured bandits), which can be extracted as a special case of the one for Markov decision processes \citep{ok2018exploration}.
While \citet{ok2018exploration} provide an implicit complexity measure $\cC(f,\cF)$ for learning any instance $f$ when knowing that it belongs to a given class $\cF$, we derive a more explicit complexity measure $\cC(f^\star,\cF_\Phi)$ for representation learning. The general lower bound follows from a fundamental result stating that any uniformly good algorithm must guarantee
$\sum_{t=1}^T\bE_{f^\star}^{\mathfrak{A}}[(f^\star(x_t,a_t) - f(x_t,a_t))^2] \geq 2\log(T)$ as $T\rightarrow \infty$ for any \emph{alternative} reward $f\in\cF$ that induces a different optimal policy than $\pi_{f^\star}^\star$. Our explicit complexity follows by leveraging a novel reformulation of the set of such \emph{alternative} rewards for representation learning which allows us to derive a closed-form expression of the above general condition.



As common in existing instance-dependent lower bounds \citep[e.g.,][]{combes2017minimal,ok2018exploration}, the complexity $\cC(f^\star,\cF_\Phi)$ is the value of an optimization problem which seeks an allocation of samples $\eta$ minimizing the regret while collecting sufficient information about the instance $f^\star$. Such an information constraint is the peculiar component in our setting as it formally establishes the minimal level of exploration that any uniformly good representation learning algorithm must guarantee. In particular, for any representation $\phi\in\Phi$, context $x\in\cX$, and sub-optimal action $a\neq\pi_{f^\star}^\star(x)$, any feasible allocation $\eta$ must guarantee 
\begin{align}\label{eq:constr}
    \underbrace{\| f^\star - F_\phi \theta^\star_\eta(\phi)\|_{D_\eta}^2}_{\text{misspecification}} + \underbrace{c_{x,a}^\eta(f^\star, \phi)}_{\text{sub-optimality}} \geq 2.
\end{align}
Here we recognize the contribution of two terms. The first one is related to the misspecification error of representation $\phi$ induced by $\eta$ (i.e., the minimum achievable mean square error when linearly estimating $f^\star$ with $\phi$ using samples collected according to $\eta$). It is trivially zero for any $\eta$ if $\phi$ is realizable. The second term is related to the complexity for learning that $a$ is a sub-optimal action for context $x$ when using representation $\phi$ to estimate the reward. Interestingly, $c_{x,a}^\eta(f^\star, \phi)$ resembles the complexity term appearing in the existing lower bound for a CLB problem with given representation $\phi$ \citep[e.g.,][]{hao2019adaptive,tirinzoni2020asymptotically}. 

The constraint requires the sum of these two terms to be large. This means that any feasible allocation $\eta$, and thus \emph{any uniformly good representation learning algorithm, must either learn that $\phi$ is misspecified or that $a$ is sub-optimal in context $x$ under the best fit of $f^\star$ with representation $\phi$}. We now discuss relevant possible cases to better undestand the complexity for achieving so.

\textbf{Case 1. $\phi$ is realizable and $z^\star_\phi(x,a)\in\im(V_\eta(\phi))$.} In this case, the misspecification term in \eqref{eq:constr} is zero for any $\eta$ and $z^\star_\phi(x,a)^\transp \theta^\star_\eta(\phi) = \Delta_{f^\star}(x,a) > 0$ by realizability, definition of $z^\star_\phi$, and sub-optimality of $a$. From \eqref{eq:constr}, this implies that $\eta$ must guarantee that $\|z^\star_\phi(x,a)\|_{V_\eta(\phi)^\dagger}^2 \leq \Delta_{f^\star}(x,a)^2/2$. It turns out that this is exactly the same complexity measure we have for learning that $(x,a)$ is sub-optimal in the CLB ($f^\star, \cF_{\{\phi\}}$).
Since $\|z^\star_\phi(x,a)\|_{V_\eta(\phi)^\dagger}$ represents the uncertainty that allocation $\eta$ has on the rewards of $(x,a)$ and $(x,\pi_{f^\star}^\star(x))$, this condition simply requires any uniformly good algorithm to reduce such uncertainty below a factor of the gap of $(x,a)$.

\textbf{Case 2. $\phi$ is realizable and $z^\star_\phi(x,a)\notin\im(V_\eta(\phi))$.} In this case, both the misspecification term and $c_{x,a}^\eta(f^\star, \phi)$ are zero. From \eqref{eq:constr}, this means that $\eta$ is infeasible. This is intuitive since, when the feature gap $z^\star_\phi(x,a)$ is not in the column space of the design matrix $V_\eta(\phi)$, the allocation $\eta$ does not provide any information about arm $a$ in the representation space of $\phi$, and thus it cannot learn whether $a$ is sub-optimal or not. Therefore, any feasible $\eta$ must guarantee $z^\star_\phi(x,a)\in\im(V_\eta(\phi))$ for all $(x,a)$ when $\phi$ is realizable, i.e., any good algorithm must explore all feature directions. This has an interesting implication: when $\spn(\{\phi(x,a)\}_{x,a}) = d_\phi$, any feasible design matrix must be invertible. This result was already proved by \cite{lattimore2017end} in the linear bandit setting using an ad-hoc derivation, while here we establish it in greater generality as a consequence of our lower bound.

\textbf{Case 3. $\phi$ is misspecified and $c_{x,a}^\eta(f^\star, \phi) = 0$.} This can happen in two cases: either $z^\star_\phi(x,a)^\transp \theta^\star_\eta(\phi) \leq 0$, which means that the sub-optimality gap of $(x,a)$ cannot be accurately estimated using representation $\phi$, or $z^\star_\phi(x,a)\notin\im(V_\eta(\phi))$. In both cases, a feasible $\eta$ must make the first term in \eqref{eq:constr} large, i.e., it must learn that $\phi$ is misspecified. Interestingly, this implies that, differently from the realizable case, a feasible allocation does not need to explore the whole feature space for $\phi$ (e.g., it does not have to make the design matrix $V_\eta(\phi)$ invertible). This is particularly relevant when $\phi$ is high-dimensional, as identifying the misspecification may be easier than covering all dimensions.

\textbf{Case 4. $\phi$ is misspecified and $c_{x,a}^\eta(f^\star, \phi) > 0$.} This is the case with most freedom: a feasible allocation can either learn that $\phi$ is misspecified or that $(x,a)$ is sub-optimal. As we shall see in Section~\ref{sec:replearn-equal}, this flexibility may be exploited to find allocations that manage to ``discard'' representations without significantly affecting the regret.

\subsection{Known-representation Case}\label{sec:clb}

%
As expected, when instantiating Theorem \ref{th:lower-bound-repr} in the standard CLB $(f^\star, \cF_{\{\phi^\star\}})$, we recover the existing lower bound for such a setting \citep{hao2019adaptive,tirinzoni2020asymptotically}.\footnote{Existing lower bounds are derived under the assumption that the full set of features $\{\phi^\star(x,a)\}_{x,a}$ span $\bR^{d^\star}$. This is without loss of generality since one can always remove redundant features by computing the low-rank SVD of $F_{\phi^\star}$.}
\begin{corollary}\label{cor:linear}
    Let $\spn(\{\phi^\star(x,a)\}_{x,a}) = d^\star$. In the CLB $(f^\star, \cF_{\{\phi^\star\}})$,
    the complexity $\cC(f^\star,\cF_{\{\phi^\star\}})$ of Theorem \ref{th:lower-bound-repr} is
    \begin{equation*}
        \begin{aligned}
        &\underset{\eta : V_\eta(\phi^\star)^{-1} \mathrm{exists}}{\inf} \sum_{x,a}\eta(x,a)\Delta_{f^\star}(x,a)
        \\ & \qquad\ \mathrm{s.t.}\quad
        \min_{x,a \neq \pi^\star_{f^\star}(x)} \frac{\Delta_{f^\star}(x,a)^2}{\|z_{\phi^\star}(x,a)\|_{V_\eta(\phi^\star)^{-1}}^2} \geq 2.
        \end{aligned}
    \end{equation*}
\end{corollary}
Comparing this result with Theorem \ref{th:lower-bound-repr}, we notice that adding one representation to the set $\Phi$ implies adding one constraint to the optimization problem, hence making the problem harder. On the positive side, Theorem \ref{th:lower-bound-repr} does not impose the strong constraint of Corollary \ref{cor:linear} for every $\phi\in\Phi$, which would require any good algorithm to learn an optimal action at every context for all representations. In fact, it may be possible to leverage the misspecification of a representation $\phi$ to lower the additional complexity w.r.t.\ the one imposed in the realizable case (see Equation \ref{eq:constr} and, e.g., Case 4 above). In Section \ref{sec:complexity}, we further elaborate on how the complexity of representation learning is impacted by these elements and how it compares with the complexity of CLBs when given a realizable representation.

\subsection{The Lower Bound is Attainable}\label{sec:lb-achievable}

It is known that instance-dependent lower bounds in the general form of \cite{ok2018exploration} can be attained. Since Theorem \ref{th:lower-bound-repr} is an instantiation of such a result, this implies that $\cC(f^\star,\cF_\Phi)$ is a \emph{tight} complexity measure for representation learning as there exist algorithms matching it.
\begin{proposition}\label{prop:lb-matchable}
    There exists an algorithm $\mathfrak{A}$ \citep[e.g.,][]{dong2022asymptotic} such that, for any representation learning problem $(f^\star,\cF_\Phi)$,
    \begin{align*}
        \limsup_{T \rightarrow \infty}\frac{\bE_{f^\star}^\mathfrak{A} \big[ R_T(f^\star) \big]}{\log(T)} \leq \cC(f^\star,\cF_\Phi).
    \end{align*}
\end{proposition}

While, to the best of our knowledge, the algorithm of \cite{dong2022asymptotic} is the only one attaining instance-optimal complexity in contextual bandits with general function classes, it is actually easy to adapt existing strategies for non-contextual bandits to our setting \citep{combes2017minimal,degenne2020structure,jun2020crush}. In particular, the algorithm of \cite{jun2020crush} would obtain an \emph{anytime} regret of order $O(\cC(f^\star,\cF_\Phi)\log(T) + \log\log(T))$. This shows that $\cC(f^\star,\cF_\Phi)$  is also a relevant finite-time complexity measure (and not only asymptotic), up to a $O(\log\log(T))$ term depending on other instance-dependent factors.
\section{Complexity of Representation Learning}\label{sec:complexity}

We now provide a series of results to better characterize the instance-dependenet complexity $\cC(f^\star,\cF_\Phi)$ of representation learning in comparison with the complexity $\cC(f^\star,\cF_{\{\phi^\star\}})$ of the single-representation CLB problem.


\subsection{Representation learning cannot be easier than learning with a given representation}\label{sec:replearn-not-easier-than-clb}

We first prove that the complexity of learning with a single representation is a lower bound for representation learning.
\begin{proposition}\label{prop:replearn-not-easier-than-clb}
    For any $\Phi$ such that $f^\star\in\cF_\Phi$, $\cC(f^\star,\cF_\Phi) \geq \sup_{\phi\in\Phi: f^\star\in\cF_{\{\phi\}}} \cC(f^\star,\cF_{\{\phi\}})$.
\end{proposition}
%


This result leverages the instance-dependent nature of the complexity derived in Theorem~\ref{th:lower-bound-repr} to compare representation learning with a single-representation CLB for \textit{every reward function} $f^\star$. This is in contrast with a worst-case analysis, where we would compare the two approaches w.r.t.\ their respective worst-case reward functions.

Whenever there is only one realizable representation $\phi^\star$ in $\Phi$, the result is intuitive, since adding misspecified representations to $\Phi$ cannot make the problem any easier. 
Nonetheless, Proposition~\ref{prop:replearn-not-easier-than-clb} has another, less obvious, implication: representation learning is at least as hard as the \textit{hardest} CLB $(f^\star, \cF_{\{\phi\}})$ among all realizable representations. More surprisingly, this result holds even when all the representations in $\Phi$ are realizable for $f^\star$. In fact, this is the unavoidable price for an algorithm to be \textit{robust} (i.e., uniformly good) to any other reward function $f'\in\mathcal F_{\Phi}$ for which some representation $\phi$ may not be realizable and it defines an intrinsic limit to the level of adaptivity to $f^\star$ that we can expect in representation learning (see Section~\ref{sec:structures} for a discussion on how this result relates to existing literature).

\subsection{There exist instances where representation learning is strictly harder than learning with a given representation}

After establishing that representation learning cannot be easier than CLBs, a natural question is: how much harder can it be? Here we show that, for any reward function $f^\star$, there exists a set of representations $\Phi$ with $f^\star \in \cF_\Phi$ such that any uniformly good representation learning algorithm must suffer regret scaling linearly with the number of contexts and actions, whereas the regret of learning with any realizable representation in the set only scales with the feature dimensionality $d \ll XA$.
\begin{proposition}\label{prop:replearn-uns}
    Let $X,A\geq 1$ and $2\leq d \leq XA$. Fix an arbitrary instance $f^\star : \cX \times \cA \rightarrow \bR$ and denote by $\Delta_{\min}$ its minimum positive gap. Then, there exists a set of $d$-dimensional representations $\Phi$ of cardinality $|\Phi| = \lceil \frac{X(A-1)}{d-1} \rceil$ such that $f^\star\in\cap_{\phi\in\Phi}\cF_{\{\phi\}}$ and
    \begin{align*}
        \cC(f^\star, \cF_\Phi) = \sum_{x\in\cX}\sum_{a\neq \pi^\star_{f^\star}(x)}\frac{2}{\Delta_{f^\star}(x,a)}. 
    \end{align*}
    Moreover, for any $\phi\in\Phi$,
    \begin{align*}
        \cC(f^\star, \cF_{\{\phi\}}) \leq \frac{2(d-1)}{\Delta_{\min}}.
    \end{align*} 
\end{proposition}
Note that the complexity $\cC(f^\star, \cF_\Phi)$ of the representation learning problem built in Proposition \ref{prop:replearn-uns} is exactly the complexity for learning the contextual bandit problem $f^\star$ when ignoring the set of representations $\Phi$, i.e., the (unstructured) tabular setting\footnote{Since the context-action space is finite, we can always run a trivial variant of UCB~\citep{auer2002finite} that estimates the reward of each $(x,a)$ independently and achieve regret $\bE_{f^\star} [ R_T(f^\star) ] \lesssim \sum_{x\in\cX}\sum_{a\neq \pi^\star_{f^\star}(x)}\frac{\log(T)}{\Delta_{f^\star}(x,a)}$. This is also the instance-optimal rate of the unstructured setting \citep{ok2018exploration}.}. Therefore, Proposition~\ref{prop:replearn-uns} proves that there exist ``hard'' representation learning problems whose complexity is the same as learning without any prior knowledge about $f^\star$. While this may be expected as the set $\Phi$ is constructed to be worst-case for $f^\star$, the second statement of Proposition~\ref{prop:replearn-uns} is more surprising. In fact, $\Phi$ is constructed using only realizable representations for $f^\star$ with dimension $d\ll XA$. As such, the complexity $\cC(f^\star, \cF_{\{\phi\}})$ for learning with any $\phi \in \Phi$ only scales (in the worst case) with $d$ and it can be arbitrarily smaller than $\cC(f^\star, \cF_\Phi)$.


\begin{remark}
    Rather than constructing a single hard instance $(f^\star,\cF_\Phi)$ where representation learning is difficult, we prove that for any reward function $f^\star$ we can find a set $\Phi$ such that $(f^\star,\cF_\Phi)$ is difficult. Hence, representation learning can be difficult regardless of the reward function.
\end{remark}

\subsection{There exist instances where representation learning is not harder than learning with a given representation}\label{sec:replearn-equal}

Unlike the previous hardness results, here we show that there exist favorable instances $(f^\star, \cF_{\Phi})$ where the complexity of representation learning is the same as the one of a CLB with a realizable representation in $\Phi$. This means that representation learning comes ``for free'' on such instances.
\begin{proposition}\label{prop:replearn-equal}
    Let $\eta^\star(x,a) = \indi{a = \pi^\star_{f^\star}(x)}$. Let $\Phi$ contain a unique realizable representation $\phi^\star$ and suppose that there exists $\epsilon > 0$ such that, for all $\phi \in \Phi$ with $f^\star\notin\cF_{\{\phi\}}$, $\| f^\star - F_\phi \theta^\star_{\eta^\star}(\phi)\|_{D_{\eta^\star}}^2 \geq \epsilon$.
    Then, $\cC(f^\star,\cF_\Phi) = \cC(f^\star,\cF_{\{\phi^\star\}})$.
\end{proposition}
Intuitively, the condition on $\Phi$ in Proposition \ref{prop:replearn-equal} requires every misspecified representation to have a minimum positive mean square error in fitting $f^\star$ when samples are collected by an optimal policy. This means that a learner is able to detect all misspecified representation by playing optimal actions, i.e., while suffering zero regret, hence making representation learning costless in the long run. 

Consider the following scheme as an example of how a simple strategy can leverage the condition Proposition \ref{prop:replearn-equal}.
Assuming finite $\Phi$, take any algorithm with sub-linear regret on the class $\cF_\Phi$, e.g., any of the algorithms for general function classes~\citep[e.g.,][]{Foster2020beyond,SimchiLevi2020falcon}. Run the algorithm in combination with an elimination rule for misspecified representations (e.g., the one proposed by \cite{banditsrl}) until only one representation remains active. When this happens, switch to playing an instance-optimal algorithm for CLBs on the remaining representation. It is easy to see that, since the starting algorithm has sub-linear regret, it plays optimal actions linearly often and, thus, thanks to the assumption in Proposition \ref{prop:replearn-equal}, it collects sufficient information to eliminate all misspecified representations in a finite time. This means that the algorithm suffers only constant regret for eliminating misspecified representations, while it never discards $\phi^\star$ with high probability. After that, playing an instance-optimal strategy on $\phi^\star$ implies that the total regret is roughly $\cC(f^\star,\cF_{\{\phi^\star\}})\log(T)$ in the long run, which is exactly the same regret we would have by running the instance optimal algorithm on $\phi^\star$ from the very beginning.

\section{Specific Representation Structures}\label{sec:structures}

We now discuss some of the specific representation learning problems studied in the literature, while providing additional insights on their instance-dependent complexity.

\subsection{Trivial Representations}

It is well known that the realizability of $\cF_\Phi$ (Assumption \ref{asm:realizability}) is crucial for efficient learning, as sub-linear regret may be impossible otherwise~\citep{Lattimore2020learninggood}. In practice, when little prior knowledge about the reward function $f^\star$ is available to design a suitable class $\cF_\Phi$, a common technique is to reduce the approximation error by expanding $\Phi$, in the hope of ensuring realizability.

When the context-arm pairs are finite, a trivial realizable representation can always be constructed as the canonical basis of $\bR^{XA}$. Let $\{(x_i,a_i)\}_{i=1}^N$ be an enumeration of all $N=XA$ context-arm pairs. Then, we can define the $XA$-dimensional features $\bar{\phi}$ as $\bar{\phi}_i(x,a) := \indi{x=x_i,a=a_i}$. It is easy to see that $f^\star = F_{\bar{\phi}}\theta$ for $\theta_i = f^\star(x_i,a_i)$.

A natural idea to build a class for representation learning is thus to start from a set $\Phi$ of ``good'' features (e.g., low dimensional or with nice spectral properties) and then add the trivial representation $\bar{\phi}$ so as to ensure realizability. The hope is that a good algorithm would still be able to leverage the ``good'' representations to achieve better results whenever possible. The following result shows that this is \emph{impossible}: every uniformly good algorithm must pay the complexity of learning without any prior knowledge on $f^\star$ as far as $\bar{\phi}$ is in the set of candidate representations.
\begin{proposition}\label{prop:replearn-trivial-rep}
    Let $\Phi$ be any set of representations (not necessarily realizable for $f^\star$). Then,
    \begin{align*}
        \cC(f^\star,\cF_{\Phi \cup \{\bar{\phi}\}})  = \sum_{x\in\cX}\sum_{a\neq \pi^\star_{f^\star}(x)}\frac{2}{\Delta_{f^\star}(x,a)}. 
    \end{align*}
\end{proposition}
As already noted in the discussion of Proposition \ref{prop:replearn-uns}, the complexity $\cC(f^\star,\cF_{\Phi \cup \{\bar{\phi}\}})$ of Proposition \ref{prop:replearn-trivial-rep} is equivalent to the complexity of learning $f^\star$ without any prior knowledge. Hence, no uniformly good algorithm can leverage the representations in $\Phi$ when $\bar{\phi}$ is also considered, no matter how good they are. For instance, the set $\Phi$ could even be a singleton $\{\phi^\star\}$ containing a realizable representation of dimension $d \ll XA$, and still representation learning over the set $\{\phi^\star, \bar{\phi}\}$ so as to achieve regret scaling with the properties of $\phi^\star$ is impossible. A similar result was derived by \cite{reda2021dealing}, who showed that learning an instance $f^\star$ which is known to be approximately linear in given features $\phi$ without knowing the amount of misspecification is as complex as learning $f^\star$ without any prior knowledge.

\subsection{Nested Features}

A popular design choice for representation learning is to be build a set of \emph{nested features} \citep{Foster2019nested,pacchiano2020stochcorral,CutkoskyDDGPP21,ghosh2021problem} $\Phi = \{\phi_1,\dots,\phi_N\}$ of increasing dimension (i.e., such that $d_i := d_{\phi_i} < d_{i+1} := d_{\phi_{i+1}}$ for all $i\in[N-1]$) that satisfy the following property: for all $i\in[N-1]$ and $(x,a)$, the first $d_i$ components of $\phi_{i+1}(x,a)$ are equal to $\phi_i(x,a)$. Let $i^\star\in[N]$ be such that $\phi_{i^\star}$ is the realizable representation of smallest dimension (which exists by Assumption \ref{asm:realizability}). The nestedness implies that $\phi_i$ is realizable for all $i \geq i^\star$.

    Several approaches have been proposed for this setting. While \citet{Foster2019nested} designed a strategy with regret 
    $\widetilde{O}(\sqrt{d_{i^\star}T} + T^{3/4})$,
    model-selection algorithms~\citep[e.g.,][]{pacchiano2020stochcorral,CutkoskyDDGPP21} achieve regret of order $\widetilde{O}(poly(N)\sqrt{d_{i^\star}T})$. Interestingly, \citet{ghosh2021problem} obtained $\widetilde{O}(\sqrt{d_{i^\star}T})$ regret, that is of the same order as the worst-case regret achievable by, e.g., LinUCB on the (unknown) smallest realizable representation $\phi_{i^\star}$.

We show that things are considerably more complex from an instance-dependent perspective. 
\begin{proposition}\label{prop:nested-no-adaptivity}
    Let $\Phi$ be a set of $N$ nested features and $f^\star\in\cF_{\Phi}$. Then,
    \begin{align}
        \cC(f^\star,\cF_\Phi) = \cC(f^\star, \cF_{\{\phi_N\}}).
    \end{align}
\end{proposition}
This result claims that representation learning on a set of nested features $\Phi$ is as difficult as a CLB problem with the realizable representation of largest dimension ($\phi_N$). This is somehow surprising since it essentially states that representation learning over nested features is useless, and one may simply learn with the highest dimensional representation (which is known to be realizable) from the very beginning. The intuition is that, while the learner knows $\phi_N$ to be realizable for any reward function (by assumption), it does not know whether this is true for $\phi_{N-1},\phi_{N-2},$ etc. Even if, say, $\phi_{N-1}$ is realizable for $f^\star$, there might be another reward function $f'$ where this is not true. Any good algorithm must explore sufficiently to eventually discriminate between $f^\star$ and $f'$ in the long run, and it turns out that the complexity for doing so is exactly $\cC(f^\star, \cF_{\{\phi_N\}})$, hence making any finer level of adaptivity impossible. Moreover, we prove in Appendix \ref{app:nested} that 
there exist problems with $d_{i^\star} \ll d_N$ where $\cC(f^\star, \cF_{\{\phi_N\}}) \gtrsim d_N$ but $\cC(f^\star, \cF_{\{\phi_{i^\star}\}}) \lesssim d_{i^\star}$. This implies that, in the worst-case, any uniformly good algorithm must suffer a dependence on the dimensionality of the largest representation, regardless of the fact that a smaller realizable representation is nested into it.

Note that this does not contradict existing results for model-selection \citep[e.g.,][]{Foster2019nested,ghosh2021problem,CutkoskyDDGPP21}. In fact, while they achieve a dependence on the worst-case regret of the best representation $\phi_{i^\star}$, they also feature some representation learning cost which dominates in the long run. For instance, the bound of \cite{ghosh2021problem} has a $O(d_N^2\log(T))$ additive term. Therefore, while some gains are possible in the small $T$ regime (e.g., scaling with $\sqrt{d_{i^\star}T}$ instead of $\sqrt{d_N T}$), in the long run the logarithmic term dominates and $\cC(f^\star, \cF_{\{\phi_N\}})\log(T)$ becomes the optimal complexity.

\subsection{HLS Representations and Sub-logarithmic Regret}

\citet{hao2019adaptive} and \citet{Papini2021leader} recently showed that in a CLB $(f^\star$, $\cF_{\{\phi\}})$ it is possible to achieve constant regret when the given realizable representation $\phi$ satisfies a certain spectral condition. 
\begin{definition}[HLS representation]\label{def:weak-hls}
    A representation $\phi$ is HLS for an instance $f^\star$ if, for all $x\in\cX$ and $a \neq \pi_{f^\star}^\star(x)$, $\phi(x,a) \in \spn(\{\phi(x,\pi_{f^\star}^\star(x))\}_{x\in\cX})$.\footnote{The original definition \citep{hao2019adaptive} requires the stronger condition $\spn(\{\phi(x,\pi_{f^\star}^\star(x))\}_{x\in\cX}) = \bR^{d_\phi}$. This is because the authors assumed that $\spn(\{\phi(x,a)\}_{x,a}) = \bR^{d_\phi}$. Here we state a generalization that works even without such an assumption.}
\end{definition}
Intuitively, a representation satisfying this property allows exploring the full feature space by playing an optimal policy. That is, playing optimal actions allows refining the reward estimates at all $(x,a)$, even those that are not played. 
Interestingly, \cite{Papini2021leader} showed that this condition is both necessary and sufficient to achieve constant regret.
\begin{theorem}[\cite{Papini2021leader}]
    Constant regret is achievable on an instance $f^\star$ if, and only if, the learner is provided with a HLS realizable representation $\phi^\star$.
\end{theorem}
When a realizable HLS representation $\phi^\star$ is not known a-priori and one must perform representation learning, it is natural to ask whether such a strong result can still be achieved. \citet{Papini2021leader,banditsrl} showed that this is indeed the case under strong conditions on $\Phi$: either \textbf{1)} all the representation are realizable or \textbf{2)} misspecified representations are detectable by any policy (i.e., such that $\min_{\theta\in\bR^{d_\phi}}\|f^\star - F_\phi\theta\|_{D_{\eta}}^2 > 0$ for any $\eta$).

We now provide a necessary and sufficient condition on the representations $\Phi$ to allow $\cC(f^\star, \cF_\Phi) = 0$. This implies that, whenever such a condition is not met, any form of sub-logarithmic regret (e.g., constant) is \emph{impossible} for \emph{any} uniformly good algorithm. On the other hand, when the condition is met, sub-logarithmic regret is achievable (and it is achieved by the algorithm mentioned in Section \ref{sec:lb-achievable}).\footnote{The best algorithm mentioned in Section \ref{sec:lb-achievable} achieves $O(\log\log(T))$ regret when $\Phi$ satisfies Proposition \ref{prop:sublog-nec-suff-explicit}. How to achieve constant regret in this setting remains an open question.}
\begin{proposition}\label{prop:sublog-nec-suff-explicit}
    A necessary and sufficient condition for $\cC(f^\star, \cF_\Phi) = 0$ is that the following two properties hold for any $\phi\in\Phi$ such that $\min_{\theta\in\bR^{d_\phi}}\|f^\star - F_\phi\theta\|_{D_{\eta^\star}}^2 = 0$ and for all $x\in\cX,a\neq\pi_{f^\star}^\star(x)$:
    \begin{enumerate}[topsep=0pt,parsep=0pt,partopsep=0pt]
        \item $z^\star_\phi(x,a)^\transp \theta^\star_{\eta^\star}(\phi) > 0$;
        \item $\phi(x,a)\in\im(V_{\eta^\star}(\phi))$.
    \end{enumerate}
\end{proposition}
Proposition \ref{prop:sublog-nec-suff-explicit} can be read as follows. Any representation $\phi$ whose misspecification is detectable by an optimal policy (i.e., such that $\min_{\theta\in\bR^{d_\phi}}\|f^\star - F_\phi\theta\|_{D_{\eta^\star}}^2 > 0$) does not bring any contribution to the regret lower bound, as already noted in Section \ref{sec:replearn-equal}. For any other representation $\phi$, the optimal policy must be able to detect that all sub-optimal pairs $(x,a)$ of $f^\star$ are indeed sub-optimal. This, in turns, requires $\phi(x,a)$ to be in the span of $V_{\eta^\star}(\phi)$ (i.e., the optimal policy explores the direction $\phi(x,a)$) and that $z^\star_\phi(x,a)^\transp \theta^\star_{\eta^\star}(\phi) > 0$ (i.e., the best approximation to the gap of $(x,a)$ using $\phi$ remains strictly positive). On the other hand, suppose that, for some $\phi\in\Phi$ with zero misspecification under an optimal policy, one of the two properties in Proposition \ref{prop:sublog-nec-suff-explicit} does not hold. Then, if $z^\star_\phi(x,a)^\transp \theta^\star_{\eta^\star}(\phi) \leq 0$, $(x,a)$ has higher reward than $(x,\pi_{f^\star}^\star(x))$ in the linear instance $(\phi,\theta^\star_{\eta^\star}(\phi))$, which means that it is impossible to learn its sub-optimality. Similarly, if $\phi(x,a)\notin\im(V_{\eta^\star}(\phi))$, an optimal policy does not explore the direction $\phi(x,a)$ at all, which means that it cannot estimate the corresponding reward. In both cases, it is necessary to repeatedly play at least some sub-optimal action, which implies that $\cC(f^\star, \cF_\Phi)  > 0$ and sub-logarithmic is thus impossible.

Perhaps surprisingly, an immediate consequence of Proposition \ref{prop:sublog-nec-suff-explicit} is that sub-logarithmic regret is \emph{impossible} if $\Phi$ contains at least one \emph{realizable} non-HLS representation.
\begin{corollary}\label{cor:sublog-impossible-realizable}
    If there exists a realizable representation $\phi\in\Phi$ which does not satisfy the HLS condition (Definition \ref{def:weak-hls}), $\cC(f^\star, \cF_\Phi) > 0$ (i.e., sub-logarithmic regret is impossible).
\end{corollary}
This implies that, even when $\Phi$ contains \emph{only} realizable representations and all but one are HLS, constant regret cannot be attained by any uniformly good algorithm.

\subsection{Fully-Realizable Representations}\label{ssec:fully.realizable}


Another specific structure is when \emph{all} representantions in $\Phi$ are realizable for all reward functions of interest. \citet{Papini2021leader} proved that, in this case, a LinUCB-based algorithm can adapt to the \emph{best} instance-dependent regret bound of a representation in $\Phi$ (e.g., it achieves constant regret when at least one representation is HLS). While our Proposition \ref{prop:replearn-not-easier-than-clb} and Corollary \ref{cor:sublog-impossible-realizable} seem to contradict their result, it turns out that \cite{Papini2021leader} consider a simpler problem: they assume the learner to be aware of $\Phi$ containing only realizable representations, while we consider the more general setting where it only knows that one of them is realizable. Intuitively, when the learner has such a strong prior knowledge, specialized strategies can be designed to achieve better results. However, such strategies would not be uniformly good on all problems in our class $\cF_\Phi$ as they do not account for misspecified representations. Therefore, the price to pay for being robust to misspecification is in general very large, no matter how ``good'' $\Phi$ is. 

As a complement to the results of \cite{Papini2021leader}, in Appendix \ref{app:fully-realizable} we show that the instance-optimal complexity of representation learning with prior knowledge about full realizability is indeed never larger than the instance-optimal complexity of a CLB with \emph{any} representation in $\Phi$, while there even exist cases where the former complexity is significantly smaller. This makes the problem of fully-realizable representation learning statistically ``easier'' than CLBs, as opposed to our general setting (Definition \ref{def:replearning}). 

\section{General Functions and Worst-case Regret}



While the instance-dependent viewpoint we considered so far allowed us to provide sharp insights on the complexity of representation learning, it is still asymptotic in nature and may ``hide'' other phenomena happening in the finite-time regime. Existing algorithms for general function classes~\citep[e.g.,][]{Foster2020beyond,SimchiLevi2020falcon} achieve $O(\sqrt{AT\log(|\cF|)})$ regret when given an arbitrary class $\cF$. This is known to be optimal in the worst possible choice of $\cF$ \citep{agarwal2012contextual}. When applied to representation learning with finite $|\Phi|$, i.e., to learn any instance $f^\star \in \cF_\Phi$ given $\cF_\Phi$, their regret bound reduces to $O(\sqrt{AT(\log(|\Phi|) + d)})$ and it is an open question whether this is optimal in the worst possible set $\Phi$. A similar result was obtained by~\citet{Moradipari2022feature}. 

In particular, one might be wondering whether a polynomial dependence on the number of actions $A$ is really unavoidable even when the learner is provided with a set of $d$-dimensional representations with $d \ll A$. The question arises mostly because some model-selection algorithms \citep{CutkoskyDDGPP21} achieve $O(\sqrt{dT\log(A)})$ regret on this problem, with some extra dependences on other problem-independent variables, like $|\Phi|$ or $T$. Such bounds give hope that adapting to the worst-case complexity of a CLB with one of the realizable representations in $\Phi$ may be possible at least in the small $T$ regime. Once again, we show that this is impossible. We state a worst-case lower bound for representation learning proving that a polynomial depedence on the number of actions is unavoidable.
\begin{theorem}\label{th:worst-case-lb-full}
    Let $N\geq 1, A\geq 4$ and $d \geq 12 \log_2(A)$. There exists a context distribution, a set of $d$-dimensional representations $\Phi$ of size $|\Phi|=N$ over $A$ arms, and a universal constant $c>0$ such that, for any learning algorithm $\mathfrak{A}$ and $T \geq \max\{\lfloor \log(dN)/\log(A) \rfloor, d/\log_2(A)\}$,
    \begin{align*}
        \sup_{f\in\cF_{\Phi}} \bE_f^\mathfrak{A}[R_T(f)] \geq c\sqrt{T \left(d\log_2(A) + A\left\lfloor \frac{\log(dN)}{\log(A)} \right\rfloor \right)}.
    \end{align*}
\end{theorem}
The proof of this result combines techniques used to derive two existing lower bounds: the $\Omega(\sqrt{dT\log(A)})$ lower bound for CLB problems of \cite{Jiahao2022reduction} and the $\Omega(\sqrt{AT\log(|\cF|)})$ lower bound for general function classes of \cite{agarwal2012contextual}. Differently from existing upper bounds that scale as $\Omega({\sqrt{Ad}})$, our lower bound decouples the polynomial dependencies on $A$ and $d$. Whether this is matchable by a specialized algorithm for representation learning, or whether existing algorithm for general function classes are already worst-case optimal in our setting, remains an intriguing open question.
\section{Discussion}\label{sec:discussion}

Our main contributions can be summarized in two fundamental hardness results. \textbf{1)} Through an instance-dependent lens, representation learning is never easier than a CLB with a given realizable representation, while the former problem can be strictly harder, up to the point that knowing that one of some given low-dimensional representations is realizable is useless. \textbf{2)} Adaptivity to the best representation is impossible in general, both in the instace-dependent long-horizon and in the worst-case small-horizon regimes. 
In particular, as opposed to worst-case results, instance-dependent adaptivity is impossible for representation learning on nested features, and the same holds when all representations are realizable if the learner does not know it a-priori. On the positive side, we characterized ``simple'' instances where representation learning is not harder than a CLB and where sub-logarithmic regret can be achieved.

Following literature on linear bandits \citep{tirinzoni2020asymptotically,kirschner2021asymptotically}, an interesting open question is how to design computationally-efficient representation learning strategies with good (e.g., worst-case optimal) finite-time regret and asymptotically instance-optimal performance.


\bibliographystyle{plainnat}
\bibliography{biblio}

\newpage
\onecolumn
\appendix

\section{Notation}\label{app:notation}

\begin{table*}[h]
    \centering
    \begin{small}
    \begin{tabular}{@{}ll@{}} 
    \toprule
    Symbol & Meaning \\
    \toprule
    \multicolumn{2}{c}{Learning problem} \\
    \cmidrule{1-2}
    $\cX$ & Finite set of $X = |\cX|$ contexts\\
    $\cA$ & Finite set of $A = |\cA|$ arms\\
    $\rho$ & Context distribution (full-support)\\
    $f^\star : \cX \times \cA \rightarrow \bR$ & Reward function\\
    $R_T(f^\star)$ & Cumulative regret when learning $f^\star$ (see Equation \ref{eq:regret})\\
    $\phi: \cX \times \cA \rightarrow \mathbb{R}^{d_\phi}$ & A $d_\phi$-dimensional representation\\
    $\Phi$ & Set of representations known to the agent\\
    $\phi^\star$ & The true realizable representation for $f^\star$\\
    $\theta^\star$ & The true parameter in $\bR^{d_{\phi^\star}}$ such that $f^\star(x,a) = \phi^\star(x,a)^\transp \theta^\star$\\
    $\cF_\Phi := \{f(\cdot,\cdot) = \phi(\cdot,\cdot)^\transp\theta \mid \phi\in\Phi,\theta\in\bR^{d_\phi}\}$ & Function class known to the agent (s.t. $f^\star \in \cF_\Phi$)\\
    $\pi^\star_f(x) = \argmax_{a\in\cA}f(x,a)$ & Optimal policy for reward function $f$\\
    $\Delta_f(x,a) := \max_{a'\in\cA}f(x,a') - f(x,a)$ & Sub-optimality gap of $(x,a)$ with reward $f$\\
    $\mathrm{KL}_{x,a}(f,f') := \frac{1}{2}(f(x,a)-f'(x,a))^2$ & KL divergence between the (Gaussian) observations in $(x,a)$ under rewards $f$ and $f'$\\
    $\eta : \cX\times\cA \rightarrow [0,\infty)$ & An allocation of samples\\
    $z^\star_\phi(x,a) := \phi(x,\pi_{f^\star}^\star(x)) - \phi(x,a)$ & Different between optimal and sub-optimal features\\
    \toprule
    \multicolumn{2}{c}{Matrix notation} \\
    \cmidrule{1-2}
    $F_\phi \in \bR^{XA \times d_\phi}$ & Feature matrix for representation $\phi$ (containing $\{\phi(x,a)\}_{x\in\cX,a\in\cA}$ as rows)\\
    $f \in \bR^{XA}$ & Vectorized reward function (s.t. $f^\star = F_{\phi^\star}\theta^\star$)\\
    $D_\eta := \diag(\{\eta(x,a)\}_{x\in\cX,a\in\cA})$ & Diagonal matrix containing allocation $\eta$\\
    $V_\eta(\phi) := F_\phi^\transp D_\eta F_\phi$ & Design matrix for representation $\phi$ (equiv. $V_\eta(\phi) := \sum_{x,a}\eta(x,a)\phi(x,a)\phi(x,a)^\transp$)\\
    $\| f^\star - F_\phi \theta\|_{D_\eta}^2$ & MSE of model $\phi(\cdot,\cdot)^\transp \theta$ in predicting $f^\star$ under allocation $\eta$ \\
    $\theta^\star_\eta(\phi) := \argmin_{\theta\in\bR^{d_\phi}} \| f^\star - F_\phi \theta\|_{D_\eta}^2$ & Best fit for $f^\star$ using representation $\phi$ and allocation $\eta$\\
    $\theta^\star_\eta(\phi) = V_\eta(\phi)^\dagger \sum_{x,a}\eta(x,a)\phi(x,a)f^\star(x,a)$ & Closed-form expression for $\theta^\star_\eta(\phi)$ (equiv. $\theta^\star_\eta(\phi) = V_\eta(\phi)^\dagger F_\phi^\transp D_\eta{f^\star}$)\\
    \toprule
    \multicolumn{2}{c}{Linear algebra} \\
    \cmidrule{1-2}
    $M^{\dagger}$ & Pseudo-inverse of a matrix $M\in\bR^{n\times m}$\\
    $\im(M)$ & Column space of a matrix $M\in\bR^{n\times m}$\\
    $\row(M)$ & Row space of a matrix $M\in\bR^{n\times m}$\\
    $\ker(M)$ & Null space of a matrix $M\in\bR^{n\times m}$\\
    $\|v\|_M^2 := v^\transp M v$ & Weighted norm for a vector $v\in\mathbb{R}^d$ and a matrix $M \in \mathbb{R}^{d\times d}$\\
    \bottomrule
    \end{tabular}
    \end{small}
    \caption{The notation adopted in this paper.}
    \label{tab:notation}
\end{table*}

\paragraph{Additional notation}

We introduce some additional terms w.r.t. those considered in the main paper. 

Given two mean-reward functions $f,f' : \cX\times\cA\rightarrow\mathbb{R}$, we define the Kullback-Leibler (KL) divergence between the corresponding distributions in a context-arm pair $(x,a)$ as $\mathrm{KL}_{x,a}(f,f') := \frac{1}{2}(f(x,a)-f'(x,a))^2$.
\section{Instance-Dependent Lower Bounds}\label{app:lower-bound}

\subsection{General Lower Bound}

We state the lower bound of \cite{ok2018exploration}, which defines a complexity measure $\cC(f^\star,\cF)$ for learning any reward $f^\star$ in any given function class $\cF$. Our lower bound for representation learning (Theorem \ref{th:lower-bound-repr}) will be derived by instantiating this result for the specific class $\cF_\Phi$.

\begin{theorem}[\cite{ok2018exploration}]\label{th:lower-bound}
    Let $\mathfrak{A}$ be a \textit{uniformly good} strategy on a class $\cF$. Then, for any $f^\star\in\cF$ such that $\pi^\star_{f^\star}$ is unique,
    \begin{equation*}
    \liminf_{T \rightarrow \infty}\frac{\bE_{f^\star}^\mathfrak{A} \big[ R_T(f^\star) \big]}{\log(T)} \geq \cC(f^\star,\cF),
    \end{equation*} 
    where $\cC(f^\star,\cF)$ is the value of the optimization problem
        \begin{equation*}
        \underset{\{\eta(x,a)\} \geq 0}{\inf} \sum_{x\in\cX}\sum_{a\in\cA}\eta(x,a)\Delta_{f^\star}(x,a)
        \quad
        \mathrm{s.t.}
        \quad
        \inf_{f \in \Lambda(f^\star,\cF)}\sum_{x\in\cX}\sum_{a\in\cA}\eta(x,a) \mathrm{KL}_{x,a}(f^\star,f)\geq 1,
        \end{equation*}
    where $\Lambda(f^\star,\cF) := \{ f \in \cF \mid \exists x\in\cX, a \neq \pi^\star_{f^\star}(x) : f(x,a) > f(x,\pi^\star_{f^\star}(x))\}$ is the set of alternatives for $f$.
\end{theorem}

\begin{proposition}[Monotonicity of $\cC(f^\star,\cF)$]\label{prop:monotone}
    For any two function classes $\underline{\cF},\overline{\cF}$ such that $\underline{\cF} \subseteq \overline{\cF}$ and $f^\star \in \underline{\cF}\cap\overline{\cF}$, $\cC(f^\star,\underline{\cF}) \leq \cC(f^\star,\overline{\cF})$.
\end{proposition}
\begin{proof}
    This result is immediate from Theorem \ref{th:lower-bound}: $\underline{\cF} \subseteq \overline{\cF}$ implies $\Lambda(f^\star,\underline{\cF}) \subseteq \Lambda(f^\star,\overline{\cF})$ by definition of alternative set, which in turns implies that the feasibility set of the optimization problem $\Lambda(f^\star,\overline{\cF})$ includes the one of the optimization problem $\Lambda(f^\star,\underline{\cF})$. The objective functions are the same, hence the second value must be smaller.
\end{proof}

\subsection{The Unstructured Case}

Let $\cF_{\mathrm{uns}}$ denote the set of all possible functions mapping $\cX\times\cA$ into $\bR$. Note that an algorithm learning an instance $f^\star$ with $\cF_{\mathrm{uns}}$ as input has no prior knowledge about $f^\star$ itself. We call this setting ``unstructured'', as opposed to the ``structured'' setting where we are given $\cF \subset \cF_{\mathrm{uns}}$ (e.g., $\cF = \cF_\Phi$). The following result formally shows that the complexity measure from Theorem \ref{th:lower-bound} for such an unstructured setting reduces exactly to the sum of inverse gaps appearing both in Proposition \ref{prop:replearn-uns} and \ref{prop:replearn-trivial-rep}.

\begin{theorem}\label{th:lower-bound-uns}
    For any $f^\star\in\cF$ such that $\pi^\star_{f^\star}$ is unique,
    \begin{align*}
        \cC(f,\cF_{\mathrm{uns}}) = \sum_{x\in\cX}\sum_{a\neq \pi^\star_{f^\star}(x)}\frac{2}{\Delta_{{f^\star}}(x,a)}.
    \end{align*}
\end{theorem}
\begin{proof}
    Note that the alternative set $\Lambda(f^\star,\cF_{\mathrm{uns}})$ can be decomposed into a union of $X(A-1)$ half-spaces, each associated to a sub-optimal context-arm pair $(x,a)$ containing those instances such that ${f}(x,a) > {f}(x,\pi^\star_{f^\star}(x))$. It is easy to see that the closest alternative for the half-space associated with $(x,a)$ (i.e., the one minimizing the KL divergence in the constraint of Theorem \ref{th:lower-bound}) is ${f}(x,a) = {f^\star}(x,a) + \Delta_{f^\star}(x,a)$ and ${f}(x',a') = {f^\star}(x',a')$ for all $(x',a')\neq(x,a)$. Therefore, the constraint associated with $(x,a)$ yields $\eta(x,a) \geq 2/\Delta_{f^\star}(x,a)^2$. There is one such constraint for each suboptimal $(x,a)$, hence yielding the stated lower bound.
\end{proof}

\subsection{Proof of Theorem \ref{th:lower-bound-repr}}

The result follows by instantiating the general lower bound of Theorem \ref{th:lower-bound} to our specific function class $\cF_\Phi := \{f(\cdot,\cdot) = \phi(\cdot,\cdot)^\transp\theta \mid \phi\in\Phi,\theta\in\bR^{d_\phi}\}$. First note that the resulting set of alternatives can be decomposed into a union of half-spaces,
\begin{align}\label{eq:alt-half-space}
    \Lambda({f^\star},\cF_\Phi) := \bigcup_{\phi\in\Phi}\bigcup_{x\in\cX}\bigcup_{a\neq \pi^\star_{f^\star}(x)} \Big\{ \theta\in\mathbb{R}^{d_\phi} \mid \phi(x,a)^\transp \theta > \phi(x,\pi^\star_{f^\star}(x))^\transp \theta \Big\}.
\end{align}
Moreover, for any ${f}\in\cF_\Phi$ such that ${f}(x,a) = \phi(x,a)^\transp \theta$,
\begin{align*}
    \sum_{x\in\cX}\sum_{a\in\cA}\eta(x,a) \mathrm{KL}_{x,a}({f^\star},{f}) = \frac{1}{2}\sum_{x\in\cX}\sum_{a\in\cA} \eta(x,a) \Big({f^\star}(x,a) - \phi(x,a)^\transp \theta\Big)^2.
\end{align*}
Therefore, the infimum over alternatives in the constraint of Theorem \ref{th:lower-bound} can be computed by performing one minimization of a quadratic function over (the closure of) each half-space in \eqref{eq:alt-half-space}. Formally, for each $\phi\in\Phi,\bar{x}\in\cX,\bar{a}\neq \pi^\star_{f^\star}(\bar x)$, we compute
\begin{equation}\label{eq:optim-half-space}
    \cI_\eta({f^\star}, \phi, \bar{x}, \bar{a}) :=  \min_{\theta \in \mathbb{R}^{d_\phi}} \frac{1}{2}\sum_{x\in\cX}\sum_{a\in\cA} \eta(x,a) \Big({f^\star}(x,a) - \phi(x,a)^\transp \theta\Big)^2 \quad \mathrm{s.t.} \quad \phi(\bar{x},\bar{a})^\transp \theta \geq \phi(\bar{x}, \pi^\star_{f^\star}(\bar x))^\transp \theta.
\end{equation}
This can be re-written in matrix notation as
\begin{align}\label{eq:optim-half-space-matrix}
    \cI_\eta({f^\star}, \phi, \bar{x}, \bar{a}) = \min_{\theta\in\bR^{d_\phi}} \frac{1}{2}\| {f^\star} - F_\phi \theta\|_{D_\eta}^2 \quad \mathrm{s.t.} \quad (\phi(\bar{x},\bar{a}) - \phi(\bar{x}, \pi^\star_{f^\star}(\bar x)))^\transp \theta \geq 0.
\end{align}
The optimal value can be found in closed-form by using Lemma \ref{lem:inf-general-halfspace} (proved below) with $z = \phi(\bar{x},\bar{a}) - \phi(\bar{x}, \pi^\star_{f^\star}(\bar x))$. Theorem \ref{th:lower-bound-repr} is then proved by plugging this result into Theorem \ref{th:lower-bound}, while noting that the constraint in the latter lower bound can be written as
\begin{align}\label{eq:constr-general}
    \cI_\eta({f^\star}, \phi, x, a) \geq 1 \ \ \forall\phi\in\Phi, x\in\cX,a\neq\pi^\star_{f^\star}(x).
\end{align}

\qed

\subsection{Implicit Formulation of Theorem \ref{th:lower-bound-repr}}

An immediate corollary from the proof of Theorem \ref{th:lower-bound-repr} (see in particular Equation \ref{eq:constr-general}) is the following implicit version of our lower bound for representation learning.

\begin{corollary}\label{cor:replearn-lb-implicit}
        For any $\Phi$ and $f^\star\in\cF_\Phi$ such that $\pi^\star_{f^\star}$ is unique,
        \begin{align*}
            \cC(f^\star,\cF_\Phi) =  \underset{\{\eta(x,a)\} \geq 0}{\inf} \sum_{x\in\cX}\sum_{a\in\cA}\eta(x,a)\Delta_{f^\star}(x,a)
            \quad
            \mathrm{s.t.}
            \quad
            \cI_\eta({f^\star}, \phi, x, a) \geq 1 \ \ \forall\phi\in\Phi, x\in\cX,a\neq\pi^\star_{f^\star}(x).
        \end{align*}
\end{corollary}

\subsection{Minimizing Over Half-Spaces}

We derive a general result which gives a closed-form expression for the minimization of the mean-square error for predicting a function $f^\star$ using the linear space $\cF_{\{\phi\}}$ subject to the constraint that the parameter $\theta$ lies into some half-space.

\begin{lemma}\label{lem:inf-general-halfspace}
    Let $\phi\in\Phi$, $z \in \bR^{d_\phi}$, and $D_\eta$ be any $XA\times XA$ diagonal matrix with non-negative entries. Consider the optimization problem
    \begin{align}\label{eq:optim-general-halfspace}
        \cI_\eta({f^\star}, \phi, z) := \min_{\theta\in\bR^{d_\phi}} \left\{ \frac{1}{2}\| {f^\star} - F_\phi \theta\|_{D_\eta}^2 \quad \mathrm{s.t.} \quad z^\transp \theta \geq 0\right\}.
    \end{align}
    Let $\theta_\eta({f^\star}, \phi, z)$ be a parameter attaining the minimum. Then,
    \begin{align*}
        \cI_\eta({f^\star}, \phi, z) &= \frac{1}{2}\| {f^\star} - F_\phi \theta^\star_\eta(\phi)\|_{D_\eta}^2 + \indi{z^\transp \theta^\star_\eta(\phi) \leq 0, z\in\im(V_\eta(\phi))} \frac{(z^\transp \theta^\star_\eta(\phi))^2}{2\|z\|_{V_\eta(\phi)^\dagger}^2},
        \\ \theta_\eta({f^\star}, \phi, z) &= \theta^\star_\eta(\phi)  - \indi{z^\transp \theta^\star_\eta(\phi) \leq 0, z\in\im(V_\eta(\phi))} \frac{z^\transp \theta^\star_\eta(\phi)}{\|z\|_{V_\eta(\phi)^\dagger}^2}V_\eta(\phi)^\dagger z.
    \end{align*}
    \end{lemma}
\begin{proof}
    Note that, for any vector $y\in\bR^{XA}$, $\|y\|_{D_\eta}^2 = \|D_\eta^{1/2}y\|_2^2$. The Lagrange dual problem corresponding to \eqref{eq:optim-general-halfspace} is
    \begin{align*}
        \max_{\lambda \in \R_{\geq 0}} \min_{\theta \in \mathbb{R}^{d_\phi}} \underbrace{ \left\{\frac{1}{2}\| D_\eta^{1/2}{f^\star} - D_\eta^{1/2}F_\phi \theta\|_{2}^2 - \lambda z^\transp \theta \right\} }_{:= {g(\theta, \lambda)}}.
    \end{align*}
    Let us fix $\lambda$ and optimize ${g(\theta, \lambda)}$ over $\theta$. Let $(U,\Sigma,V)$ be an SVD decomposition of $D_\eta^{1/2}F_\phi$, i.e., such that $D_\eta^{1/2}F_\phi = U \Sigma V^\transp$ with $U\in\bR^{XA\times XA}$, $\Sigma \in \bR^{XA\times d_\phi}$ (diagonal), and $V\in\bR^{d_\phi\times d_\phi}$. Since both $U$ and $V$ are orthogonal matrices (i.e., $U^\transp U = U U^\transp = I_{XA}$ and $V^\transp V = V V^\transp = I_{d_\phi}$),
    \begin{align*}
        {g(\theta, \lambda)} = \frac{1}{2}\| U^\transp D_\eta^{1/2}{f^\star} - \Sigma V^\transp \theta\|_{2}^2 - \lambda z^\transp V V^\transp \theta.
    \end{align*}
    Note that this follows since, for any two vectors $x,y\in\bR^{XA}$, by orthogonality of $U$,
    \begin{align*}
        \|x - Uy\|_2^2 = \|U U^\transp x - U y\|_2^2 = (U^\transp x - y)^\transp U^\transp U (U^\transp x - y) = \|U^T x - y\|_2^2.
    \end{align*}
    We can now perform a change of variables $y = V^\transp \theta$ and define the function
    \begin{align*}
        \tilde{g}(y, \lambda) = \frac{1}{2}\| U^\transp D_\eta^{1/2}{f^\star} - \Sigma y\|_{2}^2 - \lambda z^\transp V y.
    \end{align*}
    Since $V$ is invertible and $V^\transp = V^{-1}$, if $y_\lambda$ is a minimizer of $\tilde{g}(y, \lambda)$, then $\theta_\lambda = V y_\lambda$ is a minimizer of ${g(\theta, \lambda)}$.

    Let us thus minimize $\tilde{g}(y, \lambda)$ as a function of $y$. Its gradient w.r.t. $y$ is
    \begin{align*}
        \nabla_y \tilde{g}(y, \lambda) = -\Sigma^\transp (U^\transp D_\eta^{1/2}{f^\star} - \Sigma y) - \lambda V^\transp z.
    \end{align*}
    Equating it to zero, we obtain the inequality
    \begin{align*}
        \Sigma^\transp \Sigma y = \Sigma^\transp U^\transp D_\eta^{1/2}{f^\star} + \lambda V^\transp z.
    \end{align*}
    We now distinguish three cases.

    \paragraph{Case 1: $z = 0_{d_\phi}$}

    Note that $\Sigma^\transp \Sigma$ is a $d_\phi \times d_\phi$ diagonal matrix with $\{\sigma_i^2\}_{i\in[d_\phi]}$ on its diagonal, i.e., the squared singular values of $D_\eta^{1/2}F_\phi$ or, equivalently, the eigenvalues of $(D_\eta^{1/2}F_\phi)^\transp D_\eta^{1/2}F_\phi = F_\phi^\transp D_\eta F_\phi = V_\eta(\phi)$. Note that some of these might be zero as $V_\eta(\phi)$ might not be full rank. In this case, it is easy to see that a solution is
    \begin{align*}
        y_\lambda = (\Sigma^\transp \Sigma)^{\dagger} \Sigma^\transp U^\transp D_\eta^{1/2}{f^\star} = \Sigma^\dagger U^\transp D_\eta^{1/2}{f^\star},
    \end{align*}
    where we used Property 2 of the pseudo-inverse (see Appendix \ref{app:linalg}). Note that $\lambda$ has no impact on the optimization problem (in other words, there is no constraint to be satisfied). That is, the optimal $\theta$ solving the original optimization problem \eqref{eq:optim-general-halfspace} is
    \begin{align*}
        \theta = V y_\lambda = V \Sigma^\dagger U^\transp D_\eta^{1/2}{f^\star} \stackrel{(a)}{=} (D_\eta^{1/2}F_\phi)^\dagger D_\eta^{1/2}{f^\star} \stackrel{(b)}{=} (F_\phi^\transp D_\eta F_\phi)^\dagger F_\phi^\transp D_\eta{f^\star} = V_\eta(\phi)^\dagger F_\phi^\transp D_\eta{f^\star} = \theta_\eta^\star(\phi),
    \end{align*}
    where (a) uses the definition of pseudo-inverse of $D_\eta^{1/2}F_\phi$, (b) uses Property 2 of the pseudo-inverse, and the last two equalities use respectively the definition of $V_\eta(\phi)$ and of $\theta_\eta^\star(\phi)$ (see Appendix \ref{app:notation}).

    \paragraph{Case 2: $z \in \ker(V_\eta(\phi))$ and $z \neq 0_{d_\phi}$}

    Suppose that $V_\eta(\phi)$ has rank $d' < d_\phi$ and that the singular values in $\Sigma$ are sorted in non-increasing order. Then, since the columns of $V$ from index $d'+1$ to index $d_\phi$ span $\ker(V_\eta(\phi))$, the vector $V^\transp z$ has at least one non-zero element in a coordinate $i \in \{d'+1, \dots, d_\phi\}$. However, the vectors $\Sigma^\transp \Sigma y$ and $\Sigma^\transp U^\transp D_\eta^{1/2}{f^\star}$ have clearly all zero components in those coordinates. This means that the gradient cannot be equated to zero. In particular, this implies that, for any $\lambda \neq 0$, we can find $y$ such that $\tilde{g}(y,\lambda) = -\infty$. That is, the optimal solution must be at $\lambda = 0$, in which case we reduce to Case 1 (the constraint has no impact) and we get the same optimal parameter/value.

    \paragraph{Case 3: $z \notin \ker(V_\eta(\phi))$}

    In this case, $[V^\transp z]_i = 0$ for all $i \in \{d'+1, \dots, d_\phi\}$. That is, the gradient can now be equated to zero yielding, by Property 2 of the pseudo-inverse,
    \begin{align*}
        y_\lambda = (\Sigma^\transp \Sigma)^{\dagger} \Sigma^\transp U^\transp D_\eta^{1/2}{f^\star} + \lambda (\Sigma^\transp \Sigma)^{\dagger} V^\transp z = \Sigma^\dagger U^\transp D_\eta^{1/2}{f^\star} + \lambda (\Sigma^\transp \Sigma)^{\dagger} V^\transp z.
    \end{align*}
    Then,
    \begin{align*}
        \tilde{g}(y_\lambda,\lambda) = \frac{1}{2}\| U^\transp D_\eta^{1/2}{f^\star} - \Sigma \Sigma^\dagger U^\transp D_\eta^{1/2}{f^\star} - \lambda \Sigma (\Sigma^\transp \Sigma)^{\dagger} V^\transp z\|_{2}^2 - \lambda z^\transp V \Sigma^\dagger U^\transp D_\eta^{1/2}{f^\star} - \lambda^2 z^\transp V (\Sigma^\transp \Sigma)^{\dagger} V^\transp z.
    \end{align*}
    Note that $\Sigma (\Sigma^\transp \Sigma)^{\dagger} = (\Sigma^\transp)^\dagger$, $\Sigma^\dagger U^\transp D_\eta^{1/2}{f^\star} = V^\transp \theta^\star_\eta(\phi)$ (proved in Case 1), and, by Property 4 of pseudo-inverses,
    \begin{align*}
        V (\Sigma^\transp \Sigma)^{\dagger} V^\transp = V \Sigma^\dagger (\Sigma^\transp)^\dagger V^\transp = V \Sigma^\dagger U^\transp U (\Sigma^\transp)^\dagger V^\transp = (D_\eta^{1/2}F_\phi)^\dagger ((D_\eta^{1/2}F_\phi)^\transp)^\dagger = ((D_\eta^{1/2}F_\phi)^\transp D_\eta^{1/2}F_\phi)^\dagger = V_\eta(\phi)^\dagger.
    \end{align*}
    Hence,
    \begin{align*}
        \tilde{g}(y_\lambda,\lambda) &= \frac{1}{2}\| \underbrace{U^\transp D_\eta^{1/2}{f^\star} - \Sigma V^\transp \theta^\star_\eta(\phi)}_{:= w} - \lambda (\Sigma^\transp)^\dagger V^\transp z\|_{2}^2 - \lambda z^\transp \theta^\star_\eta(\phi) - \lambda^2 \|z\|_{V_\eta(\phi)^\dagger}^2
        \\ &= \frac{1}{2}\| w \|_{2}^2 + \frac{\lambda^2}{2} \underbrace{\| (\Sigma^\transp)^\dagger V^\transp z \|_2^2}_{= \|z\|_{V_\eta(\phi)^\dagger}^2} - \lambda z^\transp V \Sigma^\dagger w  - \lambda z^\transp \theta^\star_\eta(\phi) - \lambda^2 \|z\|_{V_\eta(\phi)^\dagger}^2
        \\ &= \frac{1}{2}\| w \|_{2}^2 - \lambda z^\transp \underbrace{V \Sigma^\dagger U^\transp D_\eta^{1/2}{f^\star}}_{= \theta^\star_\eta(\phi)} + \lambda z^\transp \underbrace{V \Sigma^\dagger \Sigma V^\transp \theta^\star_\eta(\phi)}_{= \theta^\star_\eta(\phi)} - \lambda z^\transp \theta^\star_\eta(\phi) - \frac{\lambda^2}{2} \|z\|_{V_\eta(\phi)^\dagger}^2
        \\ &= \frac{1}{2}\| w \|_{2}^2 - \lambda z^\transp \theta^\star_\eta(\phi) - \frac{\lambda^2}{2} \|z\|_{V_\eta(\phi)^\dagger}^2.
    \end{align*}
    Differentating w.r.t. $\lambda$,
    \begin{align*}
        \frac{\mathrm{d} \tilde{g}(y_\lambda,\lambda)}{\mathrm{d} \lambda} = -z^\transp \theta^\star_\eta(\phi) - \lambda \|z\|_{V_\eta(\phi)^\dagger}^2.
    \end{align*}
    If $z^\transp \theta^\star_\eta(\phi) \leq 0$, then the optimum is obtained with $\lambda = -z^\transp \theta^\star_\eta(\phi) / \|z\|_{V_\eta(\phi)^\dagger}^2$, otherwise the optimum is at $\lambda = 0$. Therefore, the optimal $\lambda$ is
    \begin{align*}
        \lambda = -\frac{z^\transp \theta^\star_\eta(\phi) \indi{z^\transp \theta^\star_\eta(\phi) \leq 0}}{\|z\|_{V_\eta(\phi)^\dagger}^2}.
    \end{align*}
    Plugging this into the expression of $y_\lambda$ and then into $\theta_\lambda = V y_\lambda$, we find that
    \begin{align*}
        \theta_\lambda = \theta^\star_\eta(\phi)  -  \frac{z^\transp \theta^\star_\eta(\phi) \indi{z^\transp \theta^\star_\eta(\phi) \leq 0}}{\|z\|_{V_\eta(\phi)^\dagger}^2}V(\Sigma^\transp \Sigma)^{\dagger}V^\transp z = \theta^\star_\eta(\phi)  -  \frac{z^\transp \theta^\star_\eta(\phi) \indi{z^\transp \theta^\star_\eta(\phi) \leq 0}}{\|z\|_{V_\eta(\phi)^\dagger}^2}V_\eta(\phi)^\dagger z.
    \end{align*}
    \paragraph{Final result}
    Combining the three cases, it is easy to see that the optimal parameter is
    \begin{align*}
        \theta_\eta({f^\star}, \phi, z) = \theta^\star_\eta(\phi)  - \indi{z^\transp \theta^\star_\eta(\phi) \leq 0, z\in\im(V_\eta(\phi))} \frac{z^\transp \theta^\star_\eta(\phi)}{\|z\|_{V_\eta(\phi)^\dagger}^2}V_\eta(\phi)^\dagger z.
    \end{align*}
    The corresponding optimal value is 
    \begin{align*}
        \cI_\eta({f^\star}, \phi, z) &:= \frac{1}{2}\| {f^\star} - F_\phi \theta_\eta({f^\star}, \phi, z)\|_{D_\eta}^2 
        \\ &= \frac{1}{2}\| {f^\star} - F_\phi \theta^\star_\eta(\phi)\|_{D_\eta}^2 + \indi{z^\transp \theta^\star_\eta(\phi) \leq 0, z\in\im(V_\eta(\phi))}\left( \frac{(z^\transp \theta^\star_\eta(\phi))^2}{2\|z\|_{V_\eta(\phi)^\dagger}^2} + \frac{z^\transp \theta^\star_\eta(\phi)}{\|z\|_{V_\eta(\phi)^\dagger}^2} \alpha\right),
    \end{align*}
    where
    \begin{align*}
        \alpha = z^\transp V_\eta(\phi)^\dagger F_\phi^\transp D_\eta({f^\star} - F_\phi \theta^\star_\eta(\phi))
         &= z^\transp \theta^\star_\eta(\phi) - z^\transp V_\eta(\phi)^\dagger V_\eta(\phi)\theta^\star_\eta(\phi)
        \\ &= z^\transp \theta^\star_\eta(\phi) - z^\transp V_\eta(\phi)^\dagger V_\eta(\phi)V_\eta(\phi)^\dagger F_\phi^\transp D_\eta{f^\star}
        \\ &= z^\transp \theta^\star_\eta(\phi) - z^\transp V_\eta(\phi)^\dagger F_\phi^\transp D_\eta{f^\star} = z^\transp \theta^\star_\eta(\phi) - z^\transp \theta^\star_\eta(\phi) = 0.
    \end{align*}
    This concludes the proof.

\end{proof}

\section{Proof of Propositions 2-7}\label{app:proofs}

\subsection{Proof of Proposition \ref{prop:replearn-not-easier-than-clb}}

    This result follows easily from the definition of $\cC(f^\star,\cF_\Phi)$ in Theorem \ref{th:lower-bound-repr}: for any realizable $\phi\in\Phi$ (i.e., such that $f^\star\in\cF_{\{\phi\}}$), it is sufficient to drop the constraints associated with all representations except those for $\phi$ itself. The resulting optimization problem yields exactly the lower bound for the CLB $(f^\star,\cF_{\{\phi\}})$ (Corollary \ref{cor:linear}). Then it must be that $\cC(f^\star,\cF_\Phi) \geq \cC(f^\star,\cF_{\{\phi\}})$ since we enlarged the feasibility set.

    Alternatively, one can see this as a consequence of Proposition \ref{prop:monotone} by noting that $\cF_\Phi = \cup_{\phi\in\Phi} \cF_{\{\phi\}}$. \qed

\subsection{Proof of Proposition \ref{prop:replearn-uns}}

Suppose, for simplicity, that $X(A-1)/(d-1)$ is an integer value.\footnote{The proof trivially extends to the case where $X(A-1)/(d-1)$ is not an integer by consider $\lceil X(A-1)/(d-1) \rceil$ instead.} Let $\cZ = \{(x,a)\in\cX\times\cA : a \neq \pi_{f^\star}^\star(x)\}$ be the set of sub-optimal context-arm pairs for ${f^\star}$. Note that it has cardinality $X(A-1)$. Let us partition it in $n = X(A-1)/(d-1)$ subsets $\cZ_1,\dots,\cZ_n$, each of size $d-1$. Let us enumerate the sub-optimal context-arm pairs as follows: for $i\in[n],j\in[d-1]$, $(x_{ij},a_{ij})$ denotes the $j$-th context-arm pair contained in $\cZ_i$.

Let us define a set $\Phi = \{\phi_1,\dots,\phi_n\}$ of $n$ $d$-dimensional representations as follows. For each $i\in[n]$, we choose
\begin{align*}
    \phi_i(x,a) = \begin{bmatrix}
        {f^\star}(x,a) \\
        \Delta_{f^\star}(x,a)\indi{(x,a)=(x_{i1},a_{i1})} \\
        \vdots \\
        \Delta_{f^\star}(x,a)\indi{(x,a)=(x_{i(d-1)},a_{i(d-1)})}
        \end{bmatrix}
\end{align*}
Moreover, for each sub-optimal pair $(x_{ij},a_{ij})$, let us define the function $f_{ij}$ as
\begin{align*}
    f_{ij}(x,a) = \begin{cases}
        {f^\star}(x,a) + \Delta_{{f^\star}}(x,a) & \text{if } (x,a)=(x_{ij},a_{ij})\\
        {f^\star}(x,a) & \text{otherwise}.
    \end{cases}
\end{align*}

Then, it is easy to see that $f \in \cF_{\{\phi_i\}}$ for each $f \in \{{f^\star}, f_{i1},\dots, f_{i(d-1)}\}$. In particular, ${f^\star}$ is realized by choosing $\theta = (1,0,\dots,0)^\transp$, while $f_{ij}$ with $\theta = (1,0,\dots,1,0,\dots,0)^\transp$ with the second $1$ in position $j+1$. This implies that \emph{all} $\phi\in\Phi$ are realizable for $f^\star$.

    \paragraph{Lower bounding $\cC({f^\star}, \cF_\Phi)$}

    Note that, by our choice of $\Phi$, $f_{ij} \in \cF_{\Phi}$ for all $i\in[n],j\in[d-1]$. Moreover, $f_{ij}$ corresponds to the closest alternative for the half-space associated with the suboptimal pair $(x_{ij},a_{ij})$ in the unstructured lower bound (see the proof of Theorem \ref{th:lower-bound-uns}). This implies that, if we evaluate the constraint at $(x_{ij},a_{ij})$ and $\phi_{i}$ in the implicit form of Theorem \ref{th:lower-bound-repr} (see Corollary \ref{cor:replearn-lb-implicit}), we obtain that
    \begin{align*}
        \cI_\eta({f^\star}, \phi_i, x_{ij}, a_{ij}) &= \min_{\theta: \phi_i(x_{ij},a_{ij})^\transp \theta \geq \phi_i(x_{ij},\pi_{f^\star}^\star(x_{ij}))^\transp \theta} \frac{1}{2}\sum_{x\in\cX}\sum_{a\in\cA} \eta(x,a) \Big({f^\star}(x,a) - \phi_i(x,a)^\transp \theta\Big)^2
        \\ &= \frac{1}{2}\eta(x_{ij},a_{ij})\Delta_{f^\star}(x_{ij},a_{ij}),
    \end{align*}
    since the optimum is attained by $\theta$ such that $\phi_i(x,a)^\transp \theta = f_{ij}(x,a)$. The constraint imposes that the right-hand side is larger than $1$. This holds for all sub-optimal pairs $(x_{ij},a_{ij})$. Therefore, if we call $\bar{\eta}$ an optimal feasible allocation,
    \begin{align*}
        \cC({f^\star}, \cF_\Phi) = \sum_{x\in\cX}\sum_{a\in\cA}\bar{\eta}(x,a)\Delta_{{f^\star}}(x,a) = \sum_{x\in\cX}\sum_{a\neq\pi_{f^\star}^\star(x)} \frac{2}{\Delta_{f^\star}(x,a)}.
    \end{align*}

    \paragraph{Upper bounding $\sup_{\phi\in\Phi} \cC({f^\star}, \cF_{\{\phi\}})$}

   Now take some $\phi_i\in\Phi$ for $i\in[n]$. From Corollary \ref{cor:replearn-lb-implicit}, we know that any feasible $\eta$ for the optimization problem in $\cC({f^\star}, \cF_{\{\phi_i\}})$ must satisfy
   \begin{equation*}
    \inf_{\theta: \phi_i(\bar{x},\bar{a})^\transp \theta \geq \phi_i(\bar{x},\pi_{f^\star}^\star(\bar{x}))^\transp \theta}\sum_{x\in\cX}\sum_{a\in\cA} \eta(x,a) \Big({f^\star}(x,a) - \phi_i(x,a)^\transp \theta\Big)^2 \geq 2 \quad \forall \bar{x}\in\cX,\bar{a}\neq \pi_{f^\star}^\star(\bar{x}).
    \end{equation*}
    For any $\theta$, from the definition of $\phi_i$,
    \begin{align*}
        \sum_{x\in\cX}&\sum_{a\in\cA} \eta(x,a) \Big({f^\star}(x,a) - \phi_i(x,a)^\transp \theta\Big)^2 
        \\ &= \sum_{x\in\cX}\sum_{a\in\cA} \eta(x,a) \Big({f^\star}(x,a)(1-\theta_1) - \Delta_{{f^\star}}(x,a)\sum_{j: (x,a) = (x_{ij},a_{ij})} \theta_{j+1}\Big)^2
        \\ &= \sum_{(x,a)\notin \{(x_{ij},a_{ij})\}_{j\in[d-1]}} \eta(x,a) \Big({f^\star}(x,a)(1-\theta_1) \Big)^2 + \sum_{j=1}^{d-1} \eta(x_{ij},a_{ij}) \Big({f^\star}(x_{ij},a_{ij})(1-\theta_1) - \Delta_{{f^\star}}(x_{ij},a_{ij}) \theta_{j+1}\Big)^2.
    \end{align*}
    Let us now minimize this quantity for $\theta$ such that $\phi_i(\bar{x},\bar{a})^\transp \theta \geq \phi_i(\bar{x},\pi(\bar{x}))^\transp \theta$ for some fixed $(\bar{x},\bar{a})$. 
    
    Clearly, for any $(\bar{x},\bar{a}) \notin \{(x_{ij},a_{ij})\}_{j\in[d-1]}$, we have $\phi_i(\bar{x},\bar{a})^\transp \theta = {f^\star}(\bar x, \bar a)$. This implies that there is no constraint associated with such a pair $(\bar{x},\bar{a})$ (the set of which we take the infimum is actually empty). Thus, we can set $\eta(\bar{x},\bar{a}) = 0$ for such pairs.
    
    For $(\bar{x},\bar{a}) = (x_{ij},a_{ij})$ for some $j\in[d-1]$, the condition on $\theta$ becomes $\theta_{j+1} \Delta_{f^\star}(x_{ij},a_{ij}) \geq \theta_1 \Delta_{f^\star}(x_{ij},a_{ij})$ which is equivalent to $\theta_{j+1} \geq \theta_1$ since $\Delta_{f^\star}(x_{ij},a_{ij}) > 0$. Let us suppose that $\eta(x,\pi_{f^\star}^\star(x)) = \infty$ (or a very large value) for all $x$. Note that this has no impact on the final objective value since we pay zero regret for playing optimal actions. Then, it is easy to see that the solution to the optimization problem above is attained with $\theta_1 = 1$ (otherwise the first term would be extremely large), $\theta_{j+1} = 1$, and $\theta_{l+1} = 0$ for $l\neq j$. The corresponding constraint is thus $\eta(x_{ij},a_{ij}) \geq 2/\Delta_{f^\star}(x_{ij},a_{ij})^2$ for all $j\in[d-1]$. Therefore, the resulting optimal allocation $\bar{\eta}$ requires playing (1) exactly $2/\Delta_{f^\star}(x_{ij},a_{ij})^2$ times all pairs in $\{(x_{ij},a_{ij})\}_{j\in[d-1]}$, (2) a large number of times the optimal pairs, and (3) zero times all the other pairs. Its regret is trivially
    \begin{align*}
        \sum_{x\in\cX}\sum_{a\in\cA}\bar{\eta}(x,a)\Delta_{{f^\star}}(x,a) = \sum_{j=1}^{d-1} \frac{2}{\Delta_{f^\star}(x_{ij},a_{ij})} \leq \frac{2(d-1)}{\Delta_{\min}}.
    \end{align*}
    This holds for all $i\in[n]$, which concludes the proof.

    \qed

    \subsection{Proof of Proposition \ref{prop:replearn-equal}}

        We already proved in Proposition \ref{prop:replearn-not-easier-than-clb} that $\cC(f^\star,\cF_\Phi) \geq \sup_{\phi\in\Phi: f^\star\in\cF_{\{\phi\}}} \cC(f^\star,\cF_{\{\phi\}}) =  \cC(f^\star,\cF_{\{\phi^\star\}})$, where the equality holds since $\phi^\star$ is the unique realizable representation in $\Phi$. We now prove that $\cC(f^\star,\cF_\Phi) \leq \cC(f^\star,\cF_{\{\phi^\star\}})$. Combining these two results clearly proves the statement.

        Consider any optimal feasible allocation $\bar{\eta}$ for the optimization problem associated to $\cC(f^\star, \cF_{\{\phi^\star\}})$ (i.e., the CLB with known $\phi^\star$).  Let $\tilde{\eta} = \bar{\eta} + \frac{2}{\epsilon}\eta^\star$. Clearly, since $\eta^\star$ plays only optimal actions (i.e., with zero regret),
        \begin{align*}
            \sum_{x,a}\tilde{\eta}(x,a)\Delta_{f^\star}(x,a) = \sum_{x,a}\bar{\eta}(x,a)\Delta_{f^\star}(x,a) = \cC(f^\star, \cF_{\{\phi^\star\}}).
        \end{align*}
        Moreover, for any $\phi\in\Phi$ such that $\phi\neq \phi^\star$,
        \begin{align*}
            \| f^\star - F_\phi \theta^\star_{\tilde{\eta}}(\phi)\|_{D_{\tilde{\eta}}}^2 = \min_{\theta\in\bR^{d_\phi}}\| f^\star - F_\phi \theta\|_{D_{\tilde{\eta}}}^2 \geq \frac{2}{\epsilon} \min_{\theta\in\bR^{d_\phi}}\| f^\star - F_\phi \theta\|_{D_{{\eta^\star}}}^2 = \frac{2}{\epsilon} \| f^\star - F_\phi \theta^\star_{{\eta^\star}}(\phi)\|_{D_{{\eta^\star}}}^2 \geq 2,
        \end{align*}
        where the last inequality follows by the assumption on all misspecified representations. This implies that $\tilde{\eta}$ is feasible in the optimization problem of $\cC(f^\star,\cF_\Phi)$, which, together with the equality above, proves that $\cC(f^\star,\cF_\Phi) \leq \cC(f^\star, \cF_{\{\phi^\star\}})$.

        \qed

    \subsection{Proof of Proposition \ref{prop:replearn-trivial-rep}}

    First note that $\cF_{\{\bar{\phi}\}}$ is the set of all possible reward functions. In fact, for any mapping $f : \cX \times \cA \rightarrow \bR$, $f(x,a) = \bar{\phi}(x,a)^\transp \theta$ for $\theta$ the vectorization of $f$, which implies that $f\in\cF_{\{\bar{\phi}\}}$. Then, since $\cF_{\{\bar{\phi}\}} \subseteq \cF_{\Phi \cup \{\bar{\phi}\}}$ by definition, $\cF_{\Phi \cup \{\bar{\phi}\}}$ is also the set of all possible reward functions, which implies that $\cC(f^\star,\cF_{\Phi \cup \{\bar{\phi}\}})$ is exactly the complexity of Theorem \ref{th:lower-bound-uns} for learning $f^\star$ without any prior knowledge. This is exactly the one stated in Proposition \ref{prop:replearn-trivial-rep}.
    \qed

    \subsection{Proof of Proposition \ref{prop:nested-no-adaptivity}}

    The inequality $\cC(f^\star,\cF_\Phi) \geq \cC(f^\star, \cF_{\{\phi_N\}})$ is an immediate consequence of Proposition \ref{prop:replearn-not-easier-than-clb} since $\phi_N$ is realizable. The equality follows since one could simply run an instance optimal algorithm for CLB problems \citep[e.g.,][]{tirinzoni2020asymptotically} on representation $\phi_N$ and obtain regret asymptotically approaching $\cC(f^\star, \cF_{\{\phi_N\}})\log(T)$.
    \qed

    \subsection{Proof of Proposition \ref{prop:sublog-nec-suff-explicit}}

    Note that a necessary and sufficient condition for $\cC({f^\star}, \cF_{\Phi}) = 0$ is that there exists a feasible allocation $\eta$ with $\eta(x,a)=0$ for all $x\in\cX$ and $a\neq \pi^\star_{f^\star}(x)$. That is, a feasible allocation which plays only optimal arms.

    Let us start by proving that the condition is necessary for $\cC({f^\star}, \cF_{\Phi}) = 0$. Let us proceed by contradiction. Suppose that there exists $\phi$ with $\min_\theta\|{f^\star} - F_\phi\theta\|_{D_{\eta^\star}}^2 = 0$ and $x\in\cX,a\neq\pi_{f^\star}^\star(x)$ such that $z_\phi(x,a)^\transp \theta^\star_{\eta^\star}(\phi) \leq 0$ or $z_\phi(x,a)\notin\im(V_{\eta^\star}(\phi))$, while $\cC({f^\star}, \cF_{\Phi}) = 0$. Clearly, if $\cC({f^\star}, \cF_{\Phi}) = 0$, there exists a value $M > 0$ (possibly very large) such that the allocation $\eta^\star$ scaled by $M$ is feasible . Note that scaling by a constant does not affect the column space of the resulting design matrix (i.e., $\im(V_{\eta^\star}(\phi)) = \im(V_{M\eta^\star}(\phi))$), nor does it affect the best fit in $\cF_{\{\phi\}}$ (i.e., $\theta^\star_{\eta^\star}(\phi) = \theta^\star_{M\eta^\star}(\phi)$). However, the negation of the stated condition implies that there exists $(\phi,x,a)$ such that $\cI_{M\eta^\star}({f^\star}, \phi, x, a) = \cI_{\eta^\star}({f^\star}, \phi, x, a) = 0$. Here $\cI_{\eta^\star}({f^\star}, \phi, x, a) = 0$ is the constraint function associated to $(\phi,x,a)$ in Corollary \ref{cor:replearn-lb-implicit}. That is, $M\eta^\star$ is infeasible. This is clearly a contradiction, and thus the stated condition is necessary.

Let us now prove that the condition is sufficient for $\cC({f^\star}, \cF_{\Phi}) = 0$. Take any $\phi$. We consider two cases.

\paragraph{Case 1: $\min_\theta\|{f^\star} - F_\phi\theta\|_{D_{\eta^\star}}^2 > 0$}

Simply rescaling $\eta^\star$ by $M>0$ yields $\min_\theta\|{f^\star} - F_\phi\theta\|_{D_{M\eta^\star}}^2 = M\min_\theta\|{f^\star} - F_\phi\theta\|_{D_{\eta^\star}}^2$. Since the latter term is strictly positive, we can set $M = 1 / \min_\theta\|{f^\star} - F_\phi\theta\|_{D_{\eta^\star}}^2$ to guarantee that $\cI_{M\eta^\star}({f^\star}, \phi, x, a) \geq 1$. That is $M\eta^\star$ is feasible while maintaining the objective value to zero.

\paragraph{Case 2: $\min_\theta\|{f^\star} - F_\phi\theta\|_{D_{\eta^\star}}^2 = 0$}

Take any sub-optimal $(x,a)$. We shall prove that $z_\phi(x,a)^\transp \theta^\star_{\eta^\star}(\phi) > 0$ and $z_\phi(x,a)\in\im(V_{\eta^\star}(\phi))$ imply that $M\eta^\star$ is feasible for some sufficiently large $M$. We have
\begin{align*}
    \cI_{M\eta^\star}({f^\star}, \phi, x, a) = \frac{(z_\phi(x,a)^\transp \theta^\star_{\eta^\star}(\phi))^2}{2\|z_\phi(x,a)\|_{V_{M\eta^\star}(\phi)^\dagger}^2}.
\end{align*}
Let $(U,\Sigma, V)$ be the SVD of $D_{M\eta^\star}^{1/2} F_\phi$. Note that
\begin{align*}
    \|z_\phi(x,a)\|_{V_{M\eta^\star}(\phi)^\dagger}^2 = \|z_\phi(x,a)\|_{(F_\phi^\transp D_{M\eta^\star} F_\phi)^\dagger}^2 
    & = \|z_\phi(x,a)\|_{((D_{M\eta^\star}^{1/2}F_\phi)^\transp D_{M\eta^\star}^{1/2}F_\phi)^\dagger}^2 
    \\ & = \|z_\phi(x,a)\|_{V(\Sigma^\transp\Sigma)^\dagger V^\transp}^2 
    \\ &= z_\phi(x,a)^\transp V(\Sigma^\transp\Sigma)^\dagger V^\transp z_\phi(x,a).
\end{align*}
Suppose that $V_{M\eta^\star}(\phi)$ has rank $d' \leq d_\phi$ and that the singular values in $\Sigma$ are sorted in non-increasing order. It is easy to see that the first $d'$ singular values of $\Sigma$ are $\sqrt{M}\sigma_1, \dots, \sqrt{M}\sigma_{d'}$, where $\{\sigma_i\}_{i\in[d']}$ are the first $d'$ singular values of $D_{\eta^\star}^{1/2} F_\phi$ and are all strictly positive (this is because $\rank(V_{M\eta^\star}(\phi)) = \rank(V_{\eta^\star}(\phi))$). All other singular values are zero. Thus, we have
\begin{align*}
    \|z_\phi(x,a)\|_{V_\eta(\phi)^\dagger}^2 = \sum_{i=1}^{d'} \frac{[V^\transp z_\phi(x,a)]_i^2}{M\sigma_i^2}.
\end{align*}
Recall that $\im(V_{M\eta^\star}(\phi)) = \im(V_{\eta^\star}(\phi)) = \spn(v_1,\dots,v_d')$, where $\{v_i\}_{i\in[d']}$ denote the columns of $V$ associated with a non-zero singular value in $\Sigma$. Since $z_\phi(x,a)\in\im(V_{\eta^\star}(\phi))$, the vector $V^\transp z_\phi(x,a)$ has some non-zero element among the first $d'$ components and all zeros in the remaing $d_\phi-d'$. This implies that we can set $M$ to
\begin{align*}
    M = \frac{2\sum_{i=1}^{d'} \frac{[V^\transp z_\phi(x,a)]_i^2}{\sigma_i^2}}{(z_\phi(x,a)^\transp \theta^\star_{\eta^\star}(\phi))^2},
\end{align*}
which implies that $\cI_{M\eta^\star}({f^\star}, \phi, x, a) \geq 1$ and thus $M\eta^\star$ is feasible. Overall we proved that there exists a sufficiently large $M$ such that $M\eta^\star$ satisfy all the constraints while achieving an objective value of zero. This concludes the proof.
\qed

\subsection{Proof of Corollary \ref{cor:sublog-impossible-realizable}}

The proof is trivial from Proposition \ref{prop:sublog-nec-suff-explicit}: if there exists a realizable non-HLS representation $\phi$, we have by definition that $\min_{\theta\in\bR^{d_\phi}}\|f^\star - F_\phi\theta\|_{D_{\eta}}^2 = 0$ for any $\eta$, while there exist $x\in\cX,a\neq\pi_{f^\star}^\star(x)$ such that $\phi(x,a)\in\im(V_{\eta^\star}(\phi))$. This violates the conditions of Proposition \ref{prop:sublog-nec-suff-explicit}, hence $\cC(f^\star,\cF_\Phi) > 0$.
\qed

\section{Additional Results on Nested Features}\label{app:nested}

While learning with nested features $\Phi$ cannot reduce the instance-dependent complexity with respect to learning with $\phi_N$ alone, when $d_{i^\star} \ll d_N$ one may still be wondering whether the complexity $\cC(f^\star, \cF_{\{\phi_N\}})$ can really be much larger than $\cC(f^\star, \cF_{\{\phi_{i^\star}\}})$. After all, if $f^\star$ can be described by a low-dimensional representation ($\phi_{i^\star}$) that is nested into $\phi_N$, it essentially means that $\phi_N$ is highly redundant and it might easily be compressed. We show that even in the single representation case, this is not possible in the worst case: there are instances where the regret of any uniformly good algorithm must scale with $d_N$ even though $d_{i^\star} \ll d_N$.
\begin{proposition}\label{prop:nested-dim}
    Let $2\leq d_{\min} < d_{\max} \leq X(A-1)$ and $2 \leq N \leq d_{\max} - d_{\min} + 1$. There exists an instance $f^\star : \cX \times \cA \rightarrow \bR$ with minimum positive gap $\Delta_{\min}$ and a set $\Phi = \{\phi_1,\dots,\phi_N\}$ of $N$ nested features with $\phi_N$ of dimension $d_{\max}$, $\phi_{i^\star} = \phi_1$ of dimension $d_{\min}$, and 
    \begin{align*}
        \cC(f^\star, \cF_{\{\phi_{i^\star}\}}) \leq \frac{2d_{\min}}{\Delta_{\min}}, \qquad \cC(f^\star, \cF_{\{\phi_N\}}) \geq \frac{d_{\max}}{\Delta_{\min}}. 
    \end{align*}
\end{proposition}
%
\begin{proof}
    Let $\pi : \cX \rightarrow \cA$ be any policy and define an instance $f^\star$ as $f^\star(x,\pi(x)) = \Delta_{\min}$ for all $x\in\cX$ and $f^\star(x,a) = 0$ for all $x\in\cX,a\neq\pi(x)$. Clearly, $\pi_{f^\star}^\star = \pi$ and all sub-optimal context-arm pairs have gap $\Delta_{\min}$. Let $(x_1,a_1),\dots,(x_m,a_m)$ be an arbitrary enumeration of all $m = X(A-1)$ sub-optimal context-arm pairs. 

    We define the set of $N$ nested features $\Phi = \{\phi_1,\dots,\phi_N\}$ as follows. For each $i\in[N]$, representation $\phi_i$ has dimension $d_i \in \{d_{\min}, \dots, d_{\max}\}$ with $d_{\min} \geq 2$ and $d_{\max} \leq X(A-1)+1$. In particular, $d_1 < d_2 < \dots < d_N$, $d_1 = d_{\min}$, and $d_N = d_{\max}$.  Representation $\phi_i$ is defined as
    \begin{align*}
        \phi_i(x,a) = \begin{bmatrix}
            {f^\star}(x,a) \\
            \Delta_{\min}\indi{(x,a)=(x_{1},a_{1})} \\
            \vdots \\
            \Delta_{\min}\indi{(x,a)=(x_{d_{i}-1},a_{d_{i}-1})}
            \end{bmatrix}
    \end{align*}
    That is, all representations contain the true reward function in their first component, while the $i$-th representation contains indicators over the first $d_i-1$ sub-optimal context-arm pairs in the remaining $d_i-1$ components. Note that these representations are nested and all realizable (i.e., $i^\star=1$). To prove the proposition, we only need to compute the complexities $\cC(f^\star,\cF_{\{\phi_1\}})$ and $\cC(f^\star,\cF_{\{\phi_N\}})$. For all $i\in[N]$, is it easy to see from Case 2 of the proof of Proposition \ref{prop:replearn-uns} that,
    \begin{align*}
        \cC(f^\star,\cF_{\{\phi_i\}}) = \sum_{j=1}^{d_i-1} \frac{2}{\Delta_{f^\star}(x_{j},a_{j})} = \frac{2(d_i-1)}{\Delta_{\min}}.
    \end{align*}
    Then, $\cC(f^\star,\cF_{\{\phi_1\}}) \leq \frac{2d_{\min}}{\Delta_{\min}}$ is trivial since $d_1=d_{\min}$, while $\cC(f^\star,\cF_{\{\phi_N\}}) \geq \frac{d_{\max}}{\Delta_{\min}}$ follows since $d_n = d_{\max}$ and $d_{\max} \geq 3$.
\end{proof}
Proposition \ref{prop:nested-dim} implies that achieving regret scaling with the dimensionality $d_{i^\star}$ of the smallest realizable representation is impossible in general. In the worst-case, any uniformly good algorithm must suffer a dependence on the dimensionality of the largest representation, regardless of the fact that a smaller realizable representation is nested into it.

\section{Worst-case Lower Bound (Proof of Theorem \ref{th:worst-case-lb-full})}\label{app:wc}

We start by proving two important lemmas. Then, we use them to prove a $\Omega(\sqrt{AT\log(|\Phi|)})$ lower bound (Theorem \ref{th:worst-case-lb-AlogN}) and a $\Omega(\sqrt{dT\log(A)})$ lower bound (Theorem \ref{th:worst-case-lb-dlogA}). Theorem \ref{th:worst-case-lb-full} (formally stated in Theorem \ref{th:worst-case-lb-full-app} below with precise constants) will then follow by combining these two.

\begin{lemma}\label{lem:change-measure-worst-case}
    Let $T \in \bN_{> 0}$ and denote by $\bP_{{f} | x_{1:T}}, \bE_{{f} | x_{1:T}}$ the probability and expectation operators over the full $T$-step history when learning problem ${f}$ with some fixed algorithm conditioned on the sequence of contexts $x_1,\dots,x_T$. For any couple of bandit instances ${f}_1,{f}_2 : \cX\times\cA \rightarrow \bR$ and any $(\bar x,\bar a)\in\cX\times\cA$,
    \begin{align*}
        \bE_{{f}_1 | x_{1:T}}[N_T(\bar x,\bar a)] \leq \bE_{f_2 | x_{1:T}}[N_T(\bar x,\bar a)] + \frac{N_T(\bar x)}{2} \sqrt{\sum_{x,a} \bE_{f_2 | x_{1:T}}[N_T(x,a)] (f_1(x,a) - f_2(x,a))^2}.
    \end{align*}
\end{lemma}
\begin{proof}
    By Lemma 1 of \cite{garivier2019explore}, we have that, for any random variable $Z_T$ with values in $[0,1]$ and that is measurable w.r.t. the $T$-step history conditioned on the context sequence $x_{1:T}$,
    \begin{align*}
        \mathrm{KL}(\bP_{f_2 | x_{1:T}}, \bP_{f_1 | x_{1:T}}) \geq \mathrm{kl}(\bE_{f_2 | x_{1:T}}[Z_T], \bE_{f_1 | x_{1:T}}[Z_T]),
    \end{align*}
    where $\mathrm{kl}$ denotes the KL divergence between two bernoulli distributions with parameter $\bE_{f_2 | x_{1:T}}[Z_T]$ and $\bE_{f_1 | x_{1:T}}[Z_T]$, respectively. 
    The left-hand side can be simplified as in Equation 8 of \cite{garivier2019explore} by using the chain rule of KL divergences together with the fact that rewards are Gaussian with unit variance:
    \begin{align*}
        \mathrm{KL}(\bP_{f_2 | x_{1:T}}, \bP_{f_1 | x_{1:T}}) = \frac{1}{2}\sum_{x,a} \bE_{f_2 | x_{1:T}}[N_T(x,a)] (f_1(x,a) - f_2(x,a))^2.
    \end{align*}
    On the other hand, by Pinsker's inequality,
    \begin{align*}
        \mathrm{kl}(\bE_{f_2 | x_{1:T}}[Z_T], \bE_{f_1 | x_{1:T}}[Z_T]) \geq 2 (\bE_{f_1 | x_{1:T}}[Z_T] - \bE_{f_2 | x_{1:T}}[Z_T])^2.
    \end{align*}
    Choosing $Z_T = N_T(\bar x, \bar a) / N_T(\bar x)$ and rearranging concludes the proof.
\end{proof}

\begin{lemma}[Worst-case lower bound for unstructured contextual bandits]\label{lem:wc-lb-unstructured}
    Take any $X\geq 1, A\geq 2$. Let $\Pi$ be the set of all deterministic policies mapping $[X]$ to $[A]$ and, for $\epsilon > 0$, define $\cF_\Pi := \{f_\pi : \pi\in\Pi\}$ where $f_\pi(x,a) := \epsilon\indi{a=\pi(x)}$ for all $x,a$. Then, choosing $\epsilon = \sqrt{\frac{XA}{20T}}$, for any learning algorithm $\mathfrak{A}$ and $T \geq X$,
    \begin{align*}
        \max_{{f^\star}\in\cF_\Pi} \bE_{f^\star}^\mathfrak{A}[R_T({f^\star})] \geq \frac{\sqrt{XAT}}{8\sqrt{5}}.
    \end{align*}
\end{lemma}
\begin{proof}
    The proof follows a construction of \cite{auer2002nonstochastic}. Let us consider a uniform context distribution. Take any $\pi\in\Pi$ and consider the corresponding instance $f_\pi\in\cF_\Pi$. Note that, on such an instance, each action not prescribed by $\pi$ is sub-optimal with a gap of $\epsilon$. Then, the expected regret of any algorithm $\mathfrak{A}$ is
    \begin{align*}
        \bE_{f_\pi}^\mathfrak{A}[R_T({f^\star})]
         = \epsilon \sum_{x\in\cX}\sum_{a\neq \pi(x)} \bE_{f_\pi}^\mathfrak{A}[N_T(x,a)]
        = \epsilon \left( T - \sum_{x\in\cX} \bE_{f_\pi}^\mathfrak{A}[N_T(x,\pi(x))] \right).
    \end{align*}
    We can lower bound the maximum regret over $\cF_\Pi$ as
    \begin{align*}
        \max_{{f^\star}\in\cF_\Pi} \bE_{f^\star}^\mathfrak{A}[R_T({f^\star})] = \max_{\pi\in\Pi} \bE_{f_\pi}^\mathfrak{A}[R_T({f_\pi})] \geq \frac{1}{|\Pi|}\sum_{\pi\in\Pi} \bE_{f_\pi}^\mathfrak{A}[R_T({f_\pi})] = \epsilon \left( T - \sum_{x\in\cX} \frac{1}{|\Pi|}\sum_{\pi\in\Pi} \bE_{f_\pi}^\mathfrak{A}[N_T(x,\pi(x))] \right).
    \end{align*}
    We shall thus focus on upper bounding the second term within brackets. Fix any context $\bar x$. Since the context distribution is independent of the specific instance,
    \begin{align*}
        \frac{1}{|\Pi|}\sum_{\pi\in\Pi} \bE_{f_\pi}^\mathfrak{A}[N_T(\bar x,\pi(\bar x))] = \bE_{x_{1:T}}\left[ \frac{1}{|\Pi|}\sum_{\pi\in\Pi} \bE_{f_\pi | x_{1:T}}^\mathfrak{A}[N_T(\bar x,\pi(\bar x))]\right].
    \end{align*}
    Take any $\pi$. Let $\bar{f}_\pi(\bar x, a) = 0 \ \forall a$ and $\bar{f}_\pi(x, a) = f_\pi(x,a) \ \forall x\neq \bar x, a$.   Applying Lemma \ref{lem:change-measure-worst-case} on the couple $(\bar x, \pi(\bar x))$ with $f_1 = f_\pi$ and $f_2 = \bar f_\pi$,
    \begin{align*}
        \bE_{f_\pi | x_{1:T}}^\mathfrak{A}[N_T(\bar x,\pi(\bar x))] 
        \leq \bE_{\bar f_\pi | x_{1:T}}[N_T(\bar x,\pi(\bar x))] + \frac{N_T(\bar x) \epsilon}{2} \sqrt{\bE_{\bar f_\pi | x_{1:T}}[N_T(\bar x,\pi(\bar x))]}.
    \end{align*} 
    Now take $A$ policies $\pi_1,\dots,\pi_A$ which are equal to $\pi$ in all contexts except $\bar x$, where they play each a different action. Note that one of these policies must be $\pi$ itself.
    Averaging both sides over these policies, and using Jensen's inequality,
    \begin{align*}
        \frac{1}{A}\sum_{a\in\cA} \bE_{f_{\pi_a} | x_{1:T}}^\mathfrak{A}[N_T(\bar x,\pi_a(\bar x))] 
        & \leq \frac{1}{A}\sum_{a\in\cA}\bE_{\bar f_\pi | x_{1:T}}[N_T(\bar x,\pi_a(\bar x))] + \frac{N_T(\bar x) \epsilon}{2} \sqrt{\frac{1}{A}\sum_{a\in\cA}\bE_{\bar f_\pi | x_{1:T}}[N_T(\bar x,\pi_a(\bar x))]}
        \\ &= \bE_{\bar f_\pi | x_{1:T}}\left[\frac{1}{A}\sum_{a\in\cA} N_T(\bar x,a)\right] + \frac{N_T(\bar x) \epsilon}{2} \sqrt{\bE_{\bar f_\pi | x_{1:T}}\left[\frac{1}{A}\sum_{a\in\cA} N_T(\bar x,a)\right]}
        \\ &= \frac{N_T(\bar x)}{A} + \frac{N_T(\bar x)^{3/2} \epsilon}{2\sqrt{A}},
    \end{align*}
    where the first equality holds since $\bar f_\pi$ does not depend on the choice of $\pi_a$.

    Let $\Pi_{\cX \setminus \{\bar x\}}$ be the set of all policies defined on the context space $\cX \setminus \{\bar x\}$ and, for $\pi\in\Pi_{\cX \setminus \{\bar x\}}$, let $\pi_a$ be the corresponding policy extended to $\bar x$, where $\pi_a(\bar x) = a$. We have
    \begin{align*}
        \frac{1}{|\Pi|}\sum_{\pi\in\Pi} \bE_{f_\pi | x_{1:T}}^\mathfrak{A}[N_T(\bar x,\pi(\bar x))]
         &= \frac{1}{|\Pi_{\cX \setminus \{\bar x\}}|}\sum_{\pi\in\Pi_{\cX \setminus \{\bar x\}}} \frac{1}{A} \sum_{a\in\cA}\bE_{f_{\pi_a} | x_{1:T}}^\mathfrak{A}[N_T(\bar x,\pi_a(\bar x))]
         & \leq \frac{N_T(\bar x)}{A} + \frac{N_T(\bar x)^{3/2} \epsilon}{2\sqrt{A}},
    \end{align*}
    where in the last step we used the inequality derived above. Thus, by Jensen's inequality,
    \begin{align*}
        \frac{1}{|\Pi|}\sum_{\pi\in\Pi} \bE_{f_\pi}^\mathfrak{A}[N_T(\bar x,\pi(\bar x))] 
        \leq\frac{\bE_{x_{1:T}}[N_T(\bar x)]}{A} + \frac{\bE_{x_{1:T}}[N_T(\bar x)^{3/2}] \epsilon}{2\sqrt{A}} \leq \frac{T}{XA} + \frac{\sqrt{\bE_{x_{1:T}}[N_T(\bar x)^3]} \epsilon}{2\sqrt{A}}.
    \end{align*}
    With a uniform context distribution, $N_T(\bar x) \sim \mathrm{Bin}(T,1/X)$ and its third moment has the closed-form expression
    \begin{align*}
        \bE_{x_{1:T}}[N_T(\bar x)^3] = \frac{T}{X} + \frac{3T(T-1)}{X^2} + \frac{T(T-1)(T-2)}{X^3} \leq 5\frac{T^3}{X^3}
    \end{align*}
    for $T \geq X$. Plugging everything back into our regret bound,
    \begin{align*}
        \max_{{f^\star}\in\cF_\Pi} \bE_{f^\star}^\mathfrak{A}[R_T({f^\star})] 
        \geq \epsilon \left( T - \frac{T}{A} - \epsilon\sqrt{\frac{5T^3}{4XA}} \right)
        \geq \epsilon \left( \frac{T}{2} - \epsilon\sqrt{\frac{5T^3}{4XA}} \right),
    \end{align*}
    where we used $A\geq 2$.
    The proof is concluded by optimizing over $\epsilon$.
\end{proof}

\begin{theorem}\label{th:worst-case-lb-AlogN}
    Let $N,d\geq 1, A\geq 2$.  There exist a context distribution and a set of $d$-dimensional representations $\Phi$ of size $|\Phi|=N$ over $A$ arms such that, for any learning algorithm $\mathfrak{A}$ and $T \geq \lfloor \log(dN)/\log(A) \rfloor$,
    \begin{align*}
        \max_{{f^\star}\in\cF_{\Phi}} \bE_{f^\star}^\mathfrak{A}[R_T({f^\star})] \geq \frac{\sqrt{AT\lfloor \log(dN)/\log(A) \rfloor}}{8\sqrt{5}}.
    \end{align*}
\end{theorem}
\begin{proof}
    Let $X = \lfloor \log(dN)/\log(A) \rfloor$. We shall build a set of $d$-dimensional representations $\Phi$ over $X$ contexts and $A$ arms such that $\cF_\Phi \supseteq \cF_{\Pi}$, where $\cF_{\Pi}$ is the set of functions from $[X]$ to $[A]$ defined in Lemma \ref{lem:wc-lb-unstructured}. Then, from Lemma \ref{lem:wc-lb-unstructured} we directly have
    \begin{align*}
        \max_{{f^\star}\in\cF_{\Phi}} \bE_{f^\star}^\mathfrak{A}[R_T({f^\star})] 
        \geq \max_{{f^\star}\in\cF_{\Pi}} \bE_{f^\star}^\mathfrak{A}[R_T({f^\star}, \cF_\Phi)]
        \geq
        \frac{\sqrt{XAT}}{8\sqrt{5}}.
    \end{align*}

    Recall that $\cF_\Pi := \{{f}_\pi : \pi\in\Pi\}$ where ${f}_\pi(x,a) := \epsilon\indi{a=\pi(x)}$ for all $x,a$. Suppose we want to represent the $A^X$ functions $\cF_\Pi$ in $\cF_\Phi$. Clearly, with a single representation $\phi$ we can represent at least $d$ policies $\pi_1,\dots,\pi_d$ by setting
    \begin{align*}
        \phi(x,a) = \begin{bmatrix}
            \indi{a=\pi_{1}(x)} \\
            \vdots \\
            \indi{a=\pi_{d}(x)}
          \end{bmatrix}.
    \end{align*}
    Then, the corresponding functions ${f}_{\pi_1},\dots,{f}_{\pi_d}$ are realized by choosing parameters with value $\epsilon$ on a single component and zero on all the others. In total we have $N$ feature maps, so that we can represent at least $Nd$ functions. Thus, it is enough to have $A^X \leq Nd$ to guarantee $\cF_\Phi \supseteq \cF_{\Pi}$. Rearranging this condition, we find that $X = \lfloor \log(dN)/\log(A) \rfloor$ contexts are enough. This concludes the proof.
\end{proof}

\begin{theorem}\label{th:worst-case-lb-dlogA}
    Let $N\geq 1, A\geq 4$ and $d \geq 12 \log_2(A)$. There exist a context distribution and a set of $d$-dimensional representations $\Phi$ of size $|\Phi|=N$ over $A$ arms such that, for any learning algorithm $\mathfrak{A}$ and $T \geq d/\log_2(A)$,
    \begin{align*}
        \max_{{f^\star}\in\cF_{\Phi}} \bE_{f^\star}^\mathfrak{A}[R_T({f^\star})] \geq \frac{\sqrt{Td\log_2(A)}}{16\sqrt{5}}.
    \end{align*}
\end{theorem}
\begin{proof}
    To gain intuition, let us start from $A=2$. Suppose $d$ is even and consider a set $\cX = \{x_1,\dots,x_{d/2}\}$ of $X = d/2$ contexts. Consider a problem with 2 arms $a_1,a_2$ and a $d$-dimensional representation
    \begin{align*}
        \phi(x,a) = \begin{bmatrix}
            \indi{x=x_1,a=a_1} \\
            \indi{x=x_1,a=a_2} \\
            \vdots \\
            \indi{x=x_{d/2}, a=a_1}\\
            \indi{x=x_{d/2}, a=a_2}
          \end{bmatrix}.
    \end{align*}
    It is clear that with $\phi$ we can represent any function from $\cX$ to $\{a_1,a_2\}$. Let $\Pi$ be the set of all $2^X$ policies mapping $\cX$ to $\{a_1,a_2\}$. Then, $\cF_{\{\phi\}} \supseteq \cF_{\Pi}$, where $\cF_\Pi$ is defined in Lemma \ref{lem:wc-lb-unstructured}. Thus, from Lemma \ref{lem:wc-lb-unstructured}, as far as $T \geq d/2$,
    \begin{align*}
        \max_{{f^\star}\in\cF_{\{\phi\}}} \bE_{f^\star}^\mathfrak{A}[R_T({f^\star}, \cF_{\{\phi\}})] 
        \geq
        \frac{\sqrt{dT}}{8\sqrt{5}}.
    \end{align*}
    Let us now extend this reasoning to $A \geq 2$ arms. 
    
    We use a construction inspired by \cite{Jiahao2022reduction}. Let us define a representation $\bar \phi$ of dimension $\lfloor d/\lfloor \log_2(A)\rfloor \rfloor$ for a $2$-armed problem defined over $X = \lfloor \lfloor d/\lfloor \log_2(A)\rfloor \rfloor / 2 \rfloor$ contexts and arms $\{\bar a_1, \bar a_2\}$ as
    \begin{align*}
        \forall x\in[X],a\in\{\bar a_1,\bar a_2\} : \bar \phi(x,a) = \begin{bmatrix}
            \indi{x=x_1,a=\bar a_1} \\
            \indi{x=x_1,a=\bar a_2} \\
            \vdots \\
            \indi{x=x_{X}, a=\bar a_1}\\
            \indi{x=x_{X}, a=\bar a_2}
          \end{bmatrix}.
    \end{align*}
    For $a\in [2^{\lfloor \log_2(A)\rfloor}]$, let $b(a)$ denote the binary vector of size $\lfloor \log_2(A)\rfloor$ encoding arm $a$, and let $b_i(a)$ denote its $i$-th component (such that $b_1(a)$ is the least significant digit and viceversa for $b_{\lfloor \log_2(A)\rfloor}(a))$ Then, we define the feature map for our $A$-armed problem as
    \begin{align*}
        \forall x\in[X], a\in [2^{\lfloor \log_2(A)\rfloor}] : \phi(x,a) = \begin{bmatrix}
            \bar \phi(x,\bar a_{1 + b_1(a)}) \\
            \vdots\\
            \bar \phi(x,a_{1 + b_{\lfloor \log_2(A)\rfloor}(a)})\\
            0\\
            \vdots\\
            0
          \end{bmatrix},
    \end{align*}
    where the number of zeros is $d - \lfloor d/\lfloor \log_2(A)\rfloor\rfloor \lfloor \log_2(A)\rfloor$.
    If $A$ is not a power of 2, for all remaining arms, we set $\phi(x,a)=0$.
    Intuitively, $\phi$ encodes $\lfloor \log_2(A)\rfloor$ copies of the linear bandit problem defined by $\bar{\phi}$. Moreover, selecting an action $a\in[2^{\lfloor \log_2(A)\rfloor}]$ in the problem represented by $\phi$ is equivalent to selecting actions $\{\bar a_{1+b_i(a)}\}_{i\in[\lfloor \log_2(A)\rfloor]}$ in the $\lfloor \log_2(A)\rfloor$ copies of the problem represented by $\bar \phi$. 

    Now fix some parameter $\theta\in\bR^d$ and split it into consecutive vectors $\{\theta_i\}_{i\in[\lfloor \log_2(A)\rfloor]}$ each of size $\lfloor d/\lfloor \log_2(A)\rfloor\rfloor$ (and disregard the remaining components). It is easy to see that, for any $x\in[X]$ and $a\in[2^{\lfloor \log_2(A)\rfloor}]$,
    \begin{align*}
        \max_{a'\in\cA} \phi(x,a')^\transp \theta - \phi(x,a)^\transp \theta  = \sum_{i=1}^{\lfloor \log_2(A)\rfloor} \left(\max_{a'\in\{\bar a_1, \bar a_2\}} \bar \phi(x,a')^\transp \theta_i - \bar\phi(x,\bar a_{1 + b_i(a)})^\transp \theta_i \right).
    \end{align*}
    That is, the sub-optimality gap of $(x,a)$ in the $A$-armed instance $(\phi,\theta)$ is equal to the sum of gaps of the ``binarized arms'' over the instances $\{(\bar\phi,\theta_i)\}_{i\in[\lfloor \log_2(A) \rfloor]}$. Note also that we can always choose $\theta$ such that the reward of optimal arms is strictly positive, so that the remaining $A - 2^{\lfloor \log_2(A)\rfloor}$ arms are sub-optimal.
    
    For an instance ${f^\star}$ that is linear in $\phi$ and $\theta$, let us rewrite the regret $\bE_{f^\star}^\mathfrak{A}[R_T({f^\star})] $ in the more explicit form $\bE_{\phi,\theta}^\mathfrak{A}[R_T(\phi,\theta)] $. Then, the derivation above implies that
    \begin{align*}
        \max_{\theta\in\bR^d} \bE_{\phi,\theta}^\mathfrak{A}[R_T(\phi,\theta)] 
        \geq \max_{\theta\in\bR^d} \sum_{i=1}^{\lfloor \log_2(A)\rfloor} \bE_{\bar \phi,\theta_i}^\mathfrak{A}[R_T(\bar \phi,\theta_i)]
        &= \sum_{i=1}^{\lfloor \log_2(A)\rfloor} \max_{\theta\in\bR^{\lfloor d/\lfloor \log_2(A)\rfloor\rfloor}}  \bE_{\bar \phi,\theta}^\mathfrak{A}[R_T(\bar \phi,\theta)]  
        \\ &\geq
        \frac{\lfloor \log_2(A)\rfloor\sqrt{2T\lfloor \lfloor d/\lfloor \log_2(A)\rfloor \rfloor / 2 \rfloor}}{8\sqrt{5}}
        \\ &\geq \frac{\sqrt{T(d\lfloor \log_2(A)\rfloor - 3\lfloor \log_2(A)\rfloor^2)}}{8\sqrt{5}}
        \\ &\geq \frac{\sqrt{Td\log_2(A)}}{16\sqrt{5}},
    \end{align*}
    where in the second inequality we used Lemma \ref{lem:wc-lb-unstructured} to lower bound the regret in each of the 2-armed instances using that $T \geq d/\log_2(A)$, exactly as we did in the initial example. In the third and fourth inequalities we simplified the expression using the conditions $A \geq 4$ and $d \geq 12 \log_2(A)$.
    
    Therefore, we proved that there exists a linear bandit problem (with a single representation) that satisfies the stated result. Clearly, the same applies to representation learning with $N > 1$ by simply ignoring the extra representations.
\end{proof}

\begin{theorem}\label{th:worst-case-lb-full-app}[Restatement of Theorem \ref{th:worst-case-lb-full}]
    Let $N\geq 1, A\geq 4$ and $d \geq 12 \log_2(A)$. There exist a context distribution and a set of $d$-dimensional representations $\Phi$ of size $|\Phi|=N$ over $A$ arms such that, for any learning algorithm $\mathfrak{A}$ and $T \geq \max\{\lfloor \log(dN)/\log(A) \rfloor, d/\log_2(A)\}$,
    \begin{align*}
        \max_{{f^\star}\in\cF_{\Phi}} \bE_{f^\star}^\mathfrak{A}[R_T({f^\star})] \geq \frac{\sqrt{T \left(d\log_2(A) + A\lfloor \log(dN)/\log(A) \rfloor \right)}}{32\sqrt{5}}.
    \end{align*}
\end{theorem}
\begin{proof}
    This is easy by contradiction. Suppose the statement does not hold. Then, for any set of representations and context distribution, there exists an algorithm such that the maximum regret over the family is at most the stated quantity. Using the sub-additivity of the square root followed by upper bounding the sum of the resulting two terms by twice the maximum among them, we find that such an algorithm must violate either the lower bound of Theorem \ref{th:worst-case-lb-AlogN} or the one of Theorem \ref{th:worst-case-lb-dlogA}. This is a contradiction.
\end{proof}
\section{The Fully-Realizable Case}\label{app:fully-realizable}

We provide novel insights on the complexity of representation learning in the setting studied by \cite{Papini2021leader}, where the agent knows that $f^\star$ is a linear function of \emph{all} representations $\phi\in\Phi$. That is, we consider the set of instances
\begin{align}\label{eq:class-Phi}
    \cF_\Phi^{\mathrm{FR}} := \left\{ f : \cX \times \cA \rightarrow \mathbb{R} \mid \forall \phi\in\Phi, \exists \theta\in\mathbb{R}^{d_\phi} : {f}(x,a) = \phi(x,a)^\transp \theta\ \forall x,a \right\}.
\end{align}
Clearly, $\cF_\Phi^{\mathrm{FR}} \subseteq \cF_\Phi$ and thus learning with $\cF_\Phi^{\mathrm{FR}}$ is not harder than learning with $\cF_\Phi$. This is intuitive since the agent is given more prior knowledge about ${f^\star}$ itself. 

\subsection{Instance-dependent Lower Bound}

The following result formally establishes the complexity of a representation learning problem $(f^\star,\cF_\Phi^{\mathrm{FR}})$.

\begin{theorem}\label{th:lower-bound-fr}
    Let ${f^\star}\in\cF_\Phi^{\mathrm{FR}}$ be an instance with unique optimal policy. Then, the complexity $\cC(f^\star,\cF_\Phi^{\mathrm{FR}})$ of Theorem \ref{th:lower-bound} is
    \begin{equation*}
        \cC({f^\star},\cF_\Phi^{\mathrm{FR}}) = \underset{\eta(x,a) \geq 0}{\inf} \quad \sum_{x\in\cX}\sum_{a\in\cA}\eta(x,a)\Delta_{{f^\star}}(x,a)
        \quad \mathrm{s.t.} \quad \cI_\eta({f^\star}, x, a) \geq 1 \quad x\in\cX,a\neq \pi^\star_{f^\star}(x),
    \end{equation*}
    where
    \begin{align*}
        \cI_\eta({f^\star}, x, a) := \max_{\phi\in\Phi} \frac{ \Delta_{f^\star}(x,a)^2 \indi{z_{\phi}(x,a)\in\im(V_\eta(\phi))}}{2\|\phi(x,a) - \phi(x,\pi^\star_{f^\star}(x))\|_{ V_\eta(\phi)^{\dagger}}^2}.
    \end{align*}
\end{theorem}
\begin{proof}
    It is easy to see that the constraint in Theorem \ref{th:lower-bound} decomposes into $X(A-1)$ constraints, one for each sub-optimal context-arm pair. In particular, the constraint associated with $\bar{x}\in\cX,\bar{a}\neq\pi_{f^\star}^\star(\bar{x})$ is that the following quantity is larger than one:
    \begin{equation}\label{eq:optim-lb}
        \begin{aligned}
        \inf_{{f}\in\cF_{\Phi}^{\mathrm{FR}}}\sum_{x\in\cX}\sum_{a\in\cA}\eta(x,a) \mathrm{KL}_{x,a}({f^\star},{f}) \quad \mathrm{s.t.} \quad {f}(\bar{x},\bar{a}) > {f}(\bar{x},\pi_{f^\star}^\star(\bar{x})).
        \end{aligned}
    \end{equation}
    Clearly, $\cF_{\Phi}^{\mathrm{FR}} = \cap_{\phi\in\Phi} \cF_{\{\phi\}}$. This implies that the infimum over the former set is equal to the maximum of the infima over the latter sets. This implies that the quantity above is
    \begin{align*}
        \cI_\eta({f^\star}, \bar{x}, \bar{a}) :=  \max_{\phi\in\Phi} \left\{\min_{\theta \in \mathbb{R}^{d_\phi}} \frac{1}{2}\sum_{x\in\cX}\sum_{a\in\cA} \eta(x,a) \Big({f^\star}(x,a) - \phi(x,a)^\transp \theta\Big)^2 \quad \mathrm{s.t.} \quad \phi(\bar{x},\bar{a})^\transp \theta \geq \phi(\bar{x}, \pi^\star_{f^\star}(\bar x))^\transp \theta \right\}.
    \end{align*}
    The inner problem is a minimization over a single half-space (the same as the one we compute in the realizable single-representation setting). Using Lemma \ref{lem:inf-general-halfspace} while noting that ${f^\star}(x,a) \in \cF_{\{\phi\}}$ for all $\phi\in\Phi$,
    \begin{align*}
        \cI_\eta({f^\star}, \bar{x}, \bar{a}) :=  \max_{\phi\in\Phi} \frac{ \Delta_{f^\star}(x,a)^2 \indi{z_{\phi}(x,a)\in\im(V_\eta(\phi))}}{2\|\phi(x,a) - \phi(x,\pi^\star_{f^\star}(x))\|_{ V_\eta(\phi)^{\dagger}}^2}.
    \end{align*}
\end{proof}

\subsection{Complexity of Representation Learning}

We provide a series of results to characterize the complexity of representation learning in the fully-realizable setting. In particular, we show that the problem is significantly easier than in our general setting (Assumption \ref{asm:realizability}).

\paragraph{Fully-realizable representation learning is never harder than learning with a given representation}

\begin{proposition}\label{prop:lb-fr-less-than-single}
    For any $\Phi$ such that ${f^\star} \in \cF_\Phi^{\mathrm{FR}}$, $\cC({f^\star},\cF_\Phi^{\mathrm{FR}}) \leq \cC({f^\star},\cF_{\{\phi\}}) \leq \cC({f^\star},\cF_\Phi)$.
\end{proposition}
\begin{proof}
    The second inequality is proved by Proposition \ref{prop:replearn-not-easier-than-clb} while noting that ${f^\star} \in \cF_{\{\phi\}}$ for all $\phi\in\Phi$. The first one is an immediate consequence of Theorem \ref{th:lower-bound-fr}: it is sufficient to lower bound the maximum over $\phi$ in each constraints using a single representation.
\end{proof}

\begin{remark}
    An immediate consequence of this result is that representation learning in the fully-realizable setting is never harder than a CLB with any of the representations in $\Phi$. This is in striking contrast with the general setting of Assumption \ref{asm:realizability}, where representation learning is never easier than learning with any representation in $\Phi$. The intuition from Theorem \ref{th:lower-bound-fr} is that, in the fully-realizable setting, we are allowed to choose a different representation for each $x,a$ in order to facilitate satisfying the exploration constraints, while in Theorem \ref{th:lower-bound-repr} we have one independent constraint for each representation.
\end{remark}

\paragraph{Fully-realizable representation learning can be much easier than learning with a given representation}

We present an example inspired by \cite{lattimore2017end}. In our context, learning with each single representation yields a dependence on the minimum gap, while representation learning in the fully-realizable setting does not.

\begin{proposition}\label{prop:fr-replearn-strictly-easier}
    For any $\epsilon > 0$, there exist an instance ${f^\star}$, a universal constant $c$, and a set of representations $\Phi$ such that ${f^\star}\in\cap_{\phi\in\Phi}\cF_{\{\phi\}}$ and $\cC({f^\star},\cF_{\Phi}^{\mathrm{FR}}) \leq c$, while $\min_{\phi\in\Phi} \cC({f^\star},\cF_{\{\phi\}}) \geq c/\epsilon$.
\end{proposition}
\begin{proof}
    
Let us consider a finite-armed (non-contextual) bandit problem with $4$ arms. The mean-reward vector is ${f^\star} = (1, 1-\epsilon, 1-\epsilon, 0)^\transp$. We have two realizable representations $\phi_1, \phi_2$ of dimension $d=3$ defined as
\begin{align*}
    \phi_1(a_1) = \begin{bmatrix}
        1\\
        0\\
        0
      \end{bmatrix}
      \quad
      \phi_1(a_2) = \begin{bmatrix}
        1-\epsilon\\
        \epsilon\\
        0
      \end{bmatrix}
      \quad
      \phi_1(a_3) = \begin{bmatrix}
        0\\
        0\\
        1-\epsilon
      \end{bmatrix}
      \quad
      \phi_1(a_4) = \begin{bmatrix}
        0\\
        1\\
        0
      \end{bmatrix}
\end{align*}
\begin{align*}
    \phi_2(a_1) = \begin{bmatrix}
        0\\
        0\\
        1
      \end{bmatrix}
      \quad
      \phi_2(a_2) = \begin{bmatrix}
        1-\epsilon\\
        0\\
        0
      \end{bmatrix}
      \quad
      \phi_2(a_3) = \begin{bmatrix}
        0\\
        \epsilon\\
        1-\epsilon
      \end{bmatrix}
      \quad
      \phi_2(a_4) = \begin{bmatrix}
        0\\
        1\\
        0
      \end{bmatrix}
\end{align*}
The parameter realizing ${f^\star}$ is in both cases $\theta = (1,0,1)^\transp$.

We start by computing the lower bound in the FR representation learning setting. We shall look for an upper bound to the optimal value $\cC({f^\star},\cF_{\Phi}^{\mathrm{FR}})$ which does not scale by $1/\epsilon$. Let us choose an allocation $\eta$ for which $\eta(a_1) = M$ (some very large quantity), $\eta(a_2)=\eta(a_3)=0$. We need to find the required number of pulls to $a_4$. Since we are looking for an upper bound to the optimal value, it is enough to find $\eta(a_4)$ such that $\eta$ satisfies the constraints in Theorem \ref{th:lower-bound-fr} for some specific representations (possibly different for different sub-optimal arms). We choose $\phi_1$ for $a_2$ and $a_4$, and $\phi_2$ for $a_3$. These yield the constraints
\begin{align*}
    \frac{ \Delta_{f^\star}(a_2)^2}{2\|\phi_1(a_2) - \phi_1(a_1)\|_{ V_\eta(\phi_1)^{\dagger}}^2} 
    &= \frac{ \epsilon^2}{2\|(-\epsilon,\epsilon,0)^\transp\|_{ V_\eta(\phi_1)^{\dagger}}^2} 
    \geq 1.
    \\ \frac{ \Delta_{f^\star}(a_3)^2}{2\|\phi_2(a_3) - \phi_2(a_1)\|_{ V_\eta(\phi_2)^{\dagger}}^2} 
    &= \frac{ \epsilon^2}{2\|(0,\epsilon,-\epsilon)^\transp\|_{ V_\eta(\phi_2)^{\dagger}}^2} 
    \geq 1.
    \\ \frac{ \Delta_{f^\star}(a_4)^2}{2\|\phi_1(a_4) - \phi_1(a_1)\|_{ V_\eta(\phi_1)^{\dagger}}^2} 
    &= \frac{ 1}{2\|(-1,1,0)^\transp\|_{ V_\eta(\phi_1)^{\dagger}}^2} 
    \geq 1.
\end{align*}
Note that, by our choice of $\eta$, $V_\eta(\phi_1)$ and $V_\eta(\phi_2)$ are diagonal matrices with diagonal elements $(M, \eta(a_4), 0)$ and $(0, \eta(a_4), M)$, respectively. Therefore, the constraints above reduce to
\begin{align*}
    \frac{ \epsilon^2}{\epsilon^2/M + \epsilon^2/\eta(a_4)} 
    \geq 2, 
    \quad \frac{ \epsilon^2}{\epsilon^2/\eta(a_4) + \epsilon^2/M} 
    \geq 2, 
    \quad \frac{1}{1/M + 1/\eta(a_4)} 
    \geq 2.
\end{align*}
Letting $M$ go to infinity (which does not alter the objective value), we find that $\eta(a_4) \geq 2$ suffices. Therefore, we proved that the allocation $\eta = (\infty,0,0,2)^\transp$ is feasible, and thus the optimal value is bounded by $\cC({f^\star},\cF_{\Phi}^{\mathrm{FR}}) \leq 2$.

We now show that the regret when learning with each of the single representations scales at least by $1/\epsilon$. Let us do it for $\phi_1$. For $\phi_2$ the argument will be the same since the two representations are equal up to a permutation of the first and third component.

Clearly, since we want to lower bound the optimal value $\cC({f^\star},\cF_{\{\phi_1\}})$, we can drop all constraints but the one associated with $a_3$, i.e.,
\begin{align*}
    \frac{ \Delta_{f^\star}(a_3)^2}{2\|\phi_1(a_3) - \phi_1(a_1)\|_{ V_\eta(\phi_1)^{\dagger}}^2} 
    = \frac{ \epsilon^2}{2\|(-1,0,1-\epsilon)^\transp\|_{ V_\eta(\phi_1)^{\dagger}}^2} 
    \geq 1.
\end{align*}

Let us set once again $\eta(a_1) = M$ (some very large value). Since the constraint associated with $a_3$ requires to make the feature norm of $\phi_1(a_3)-\phi_1(a_1) = (-1,0,1-\epsilon)^\transp$ small, clearly both $a_2$ and $a_4$ do not serve to this purpose (they do not cover the third dimension). So the optimal strategy must have $\eta(a_2)=\eta(a_4)=0$. The matrix $V_\eta(\phi_1)^{\dagger}$ is then diagonal with elements $(M,0,(1-\epsilon)^2 \eta(a_3))$. Thus, the constraint reduces to
\begin{align*}
    \frac{ \epsilon^2}{1/M + 1/\eta(a_3)} 
    \geq 2.
\end{align*}
This implies that $\eta(a_3) \geq 2/\epsilon^2$. Plugging this into the regret, recalling that action $a_3$ has gap $\epsilon$, we obtain that $\cC({f^\star},\cF_{\{\phi_1\}}) \geq 2/\epsilon$.

\end{proof}

\subsection{Necessary and Sufficient Condition for Constant Regret}

\begin{proposition}\label{prop:sublog-nec-suff-fr}
    A necessary and sufficient condition for $\cC({f^\star},\cF_\Phi^{\mathrm{FR}}) = 0$ is that, for all $x\in\cX,a\neq\pi_{f^\star}^\star(x)$, there exists $\phi\in\Phi$ such that $z_\phi(x,a)\in\im(V_{\eta^\star}(\phi))$ (equiv. $\phi(x,a)\in\im(V_{\eta^\star}(\phi))$).
\end{proposition}
\begin{remark}
    This result shows that the mixing HLS condition assumed by \cite{Papini2021leader}  (see their Definition 1) is actually necessary for constant regret.
\end{remark}
\begin{proof}
    Proving that the condition is necessary can be easily done by contradiction. If $\cC({f^\star},\cF_\Phi^{\mathrm{FR}}) = 0$, then a rescaling of $\eta^\star$ must be feasible. However, if for some sub-optimal $(x,a)$ we have $z_\phi(x,a)\notin\im(V_{\eta^\star}(\phi))$ for all $\phi$, that would imply that any rescaling of $\eta^\star$ is actually infeasible, hence yielding a contradiction.

    The proof that the condition is sufficient can be done as a simple extension of the one of Proposition \ref{prop:sublog-nec-suff-explicit}. Simply take any sub-optimal $(x,a)$ and show that a re-scaling of $\eta^\star$ is feasible by following exactly the same steps as in Proposition \ref{prop:sublog-nec-suff-explicit}.
\end{proof}
\section{Useful Linear Algebra Results}\label{app:linalg}

\subsection{Singular value decomposition}

We recall that the SVD of a real matrix $A \in \bR^{n \times m}$ is a factorization of the form $A = U\Sigma V^\transp$, with $U\in \bR^{n \times n}$ orthogonal (i.e., such that $U^\transp U = UU^\transp = I)$, $\Sigma \in \bR^{n \times m}$ diagonal, and $V \in \bR^{m \times m}$ orthogonal. Suppose $A$ has rank $d \leq \min\{n,m\}$ and that the diagonal entries of $\Sigma$ (i.e., the singular values of $A$) are in decreasing order ($\sigma_1 \geq \dots \geq \sigma_{\min\{n,m\}} \geq 0$). We list some well-known properties of the SVD decomposition.

\paragraph{Properties}
\begin{itemize}
    \item The number of non-zero entries in $\Sigma$ correspond to the rank of $A$. 
    \item Let $u_1, \dots, u_d$ be the columns of $U$ (i.e., the left singular vectors) corresponding to non-zero singular values. Then $\mathrm{span}(u_1, \dots, u_d) = \mathrm{Im}(A)$.
    \item Let $v_1, \dots, v_d$ be the columns of $V$ (i.e., the right singular vectors) corresponding to non-zero singular values. Then $\mathrm{span}(v_1, \dots, v_d) = \row(A)$.
    \item Let $v_{d+1}, \dots, v_{\min\{n,m\}}$ be the columns of $V$ (i.e., the right singular vectors) corresponding to zero singular values. Then $\mathrm{span}(v_{d+1}, \dots, v_{\min\{n,m\}}) = \mathrm{Ker}(A)$.
\end{itemize}

Let $A \in \bR^{n \times m}$ with $n \geq m$ and $\rank(A) = d$. Then,
\begin{itemize}
    \item $\rank(A^T A) = \rank(A)$.
    \item $\im(A^T A) = \row(A) = \spn(v_1, \dots, v_d)$.
    \item $\ker(A^T A) = \ker(A) = \mathrm{span}(v_{d+1}, \dots, v_{\min\{n,m\}})$.
\end{itemize}

\subsection{Pseudo-inverse}

We recall that the pseudo-inverse of a matrix $A \in \bR^{n \times m}$ is defined as $A^\dagger = V\Sigma^\dagger U^\transp$, where $(U,\Sigma,V)$ is the SVD of $A$ and $\Sigma^\dagger$ is a diagonal matrix with the inverse of the non-zero elements of $\Sigma$.

\paragraph{Properties}

\begin{enumerate}
    \item $A A^\dagger A = A$ and $A^\dagger A A^\dagger = A^\dagger$.
    \item $(A^\transp A)^\dagger A^T = A^\dagger$.
    \item $(A^\dagger)^\transp = (A^\transp)^\dagger$.
    \item $(A A^\transp)^\dagger = (A^T)^\dagger A^\dagger$.
    \item If either $A^\transp A = I$ or $B B^\transp = I$ or $A = B^\transp$: $(AB)^\dagger = B^\dagger A^\dagger$.
\end{enumerate}

\end{document}